\documentclass{article}

    \PassOptionsToPackage{numbers, compress}{natbib}

 \usepackage[main, final]{neurips_2026}

\usepackage[utf8]{inputenc} 
\usepackage[T1]{fontenc}    
\usepackage{hyperref}       
\usepackage{url}            
\usepackage{booktabs}       
\usepackage{amsfonts}       
\usepackage{nicefrac}       
\usepackage{microtype}      
\usepackage{xcolor}         

\usepackage{graphicx}         

\usepackage{amsmath}
\usepackage{amssymb}
\usepackage{amsthm}
\usepackage{mathtools}
\usepackage{stmaryrd}
\usepackage{wrapfig}

\newtheorem{definition}{Definition}
\newtheorem{theorem}{Theorem}
\newtheorem{corollary}{Corollary}[theorem]

\title{Intersectional Fairness via Mixed-Integer Optimization}

\author{%
  Jiří Němeček \\
  Czech Technical University \\
  \texttt{contact@nemecekjiri.cz} \\
  \And
  Mark Kozdoba \\
  Technion\\
  \And
  Illia Kryvoviaz \\ 
  Czech Technical University\\
  \And
  Tomáš Pevný \\
  Czech Technical University\\
  \And
  Jakub Mareček \\
  Czech Technical University\\  
}

\begin{document}

\maketitle

\begin{abstract}
  The deployment of Artificial Intelligence in high-risk domains, such as finance and healthcare, necessitates models that are both fair and transparent. While regulatory frameworks, including the EU's AI Act, mandate bias mitigation, they are deliberately vague about the definition of bias. In line with existing research, we argue that true fairness requires addressing bias at the intersections of protected groups. We propose a unified framework that leverages Mixed-Integer Optimization (MIO) to train intersectionally fair and intrinsically interpretable classifiers. We prove the equivalence of two measures of intersectional fairness (MSD and SPSF) in detecting the most unfair subgroup and empirically demonstrate that our MIO-based algorithm improves performance in finding bias. We train high-performing, interpretable classifiers that bound intersectional bias below an acceptable threshold, offering a robust solution for regulated industries and beyond.
\end{abstract}

\section{Introduction}
The integration of Artificial Intelligence (AI) systems into sensitive domains—ranging from lending and hiring to healthcare diagnostics—has shown potential for society, but simultaneously brought forward some critical issues.
One of these is the pervasive problem of algorithmic bias, in which systematic errors in data or model design lead to the unequal allocation of resources or opportunities to protected demographic groups.

Recognizing the gravity of these ethical and societal challenges, regulatory bodies and industry stakeholders have responded with various frameworks on bias mitigation. These efforts include specific legislation, such as the European Union's AI Act \citep{ai-act} and Codes of Practice, along with the proliferation of technical standards \citep{schwartz_towards_2022, 10851955}, codes of conduct,\footnote{This includes major companies, such as \href{https://learn.microsoft.com/en-us/legal/ai-code-of-conduct}{Microsoft} or \href{https://ai.google/principles/}{Google}.} and certifications\footnote{For example, \url{https://www.projectallyai.org/}} mandating the detection and mitigation of AI bias.

\begin{figure}
    \centering
    \includegraphics[width=0.6\linewidth]{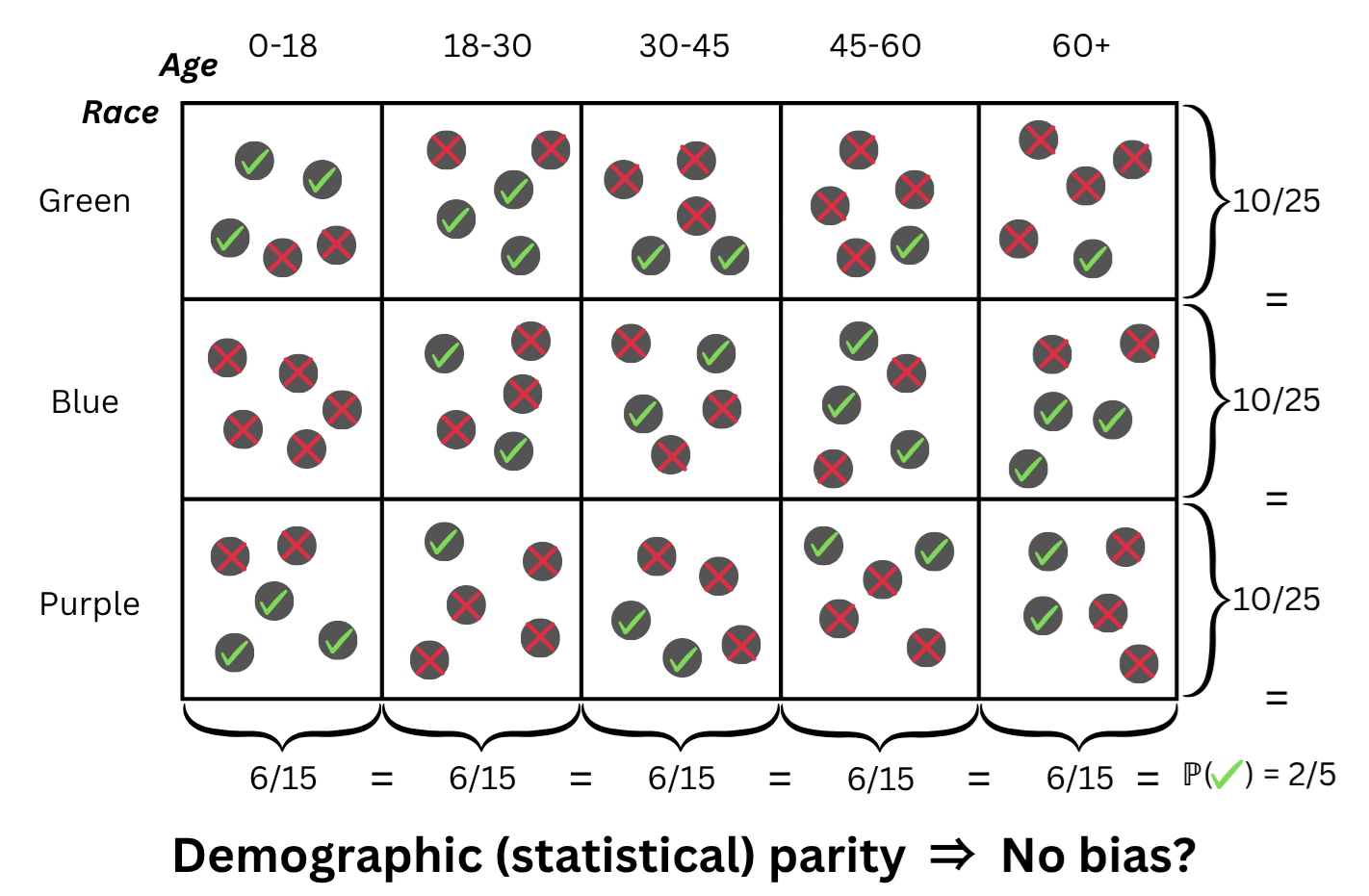}
    \caption{An illustrative example of when marginal bias is not enough. Statistical parity is achieved, suggesting marginal fairness, while blue people aged 0-18 are disproportionately more rejected. The graphic is a modification of Figure 1 in \citet{nemecek_bias_2025}.}
    \label{fig:motivation}
\end{figure}

However, following the requirements is complicated by the fact that no definition of fairness is universally recommended. Instead, compliance frameworks are typically vague, with statements like ``[data shall undergo] appropriate measures to detect, prevent and mitigate possible biases'' \citep[Article 10]{ai-act}. This leaves the developers to choose from a multitude of established mathematical fairness criteria \citep{mehrabi_survey_2021}. Despite this variety, most measures focus solely on single-attribute group fairness metrics (i.e., marginal fairness), which is fundamentally insufficient. As illustrated by Figure \ref{fig:motivation}, there is potential for hidden biases within the intersection of groups (subgroups) when marginal fairness is achieved.
\citet{buolamwini_gender_2018} showed disparate performance of commercial image classifiers across gender and skin tones. The disparity of the intersectional subgroups was more pronounced than simply adding up their marginal disparities. Yet, despite this seminal finding, recent bias audits mostly ignore the intersectional aspects of fairness \citep{gerchick_auditing_2025}.

Practical scenarios require considering intersectional fairness, that is, accounting for compounded discrimination against uniquely disadvantaged subgroups defined by combinations of sensitive attributes. This means using measures of intersectional bias, such as Statistical Parity Subgroup Fairness \citep{kearns_preventing_2018} or many others \citep{foulds_intersectional_2020,ghosh_characterizing_2021,gohar_survey_2023,nemecek_bias_2025}. 
Moreover, since the number of intersections depends exponentially on the number of protected attributes, studying all subgroups separately might become computationally infeasible \citep{nemecek_bias_2025}.
In such cases, rather than evaluating bias for each subgroup, one would like to \emph{detect} subgroups with a certain level of bias. 
Still, there are subgroups for which an insufficient amount of data prevents a reliable evaluation of bias due to the curse of dimensionality. 
While neglecting small-sample subgroups shows this technical approach's inadequacy for mitigating intersectional bias \citep{kong_are_2022}, there is still merit in evaluating all subgroups that can be evaluated.

Crucially, in many high-stakes application areas, the requirement for fairness is intrinsically paired with a demand for transparency. 
Fairness and understandability should thus be explored jointly \citep{alvarez_policy_2024}.
As argued by \citet{rudin_stop_2019}, domains requiring high accountability should move beyond unreliable explanations and prioritize intrinsic interpretability. 


\paragraph{Main Contributions} In this paper, we
\begin{itemize}
    \item prove the equivalence of SPSF \citep{kearns_preventing_2018} and MSD \citep{nemecek_bias_2025} in detecting the most biased subgroup, and compare them empirically;
    \item improve on existing work in finding the most unfair groups according to the SPSF measure;
    \item propose and validate a training algorithm for certifiably fair and interpretable classifiers under measures of intersectional fairness; 
\end{itemize}

\paragraph{Notation} Throughout the paper, let $\mathcal{D}$ be a probability distribution over the input space $\mathcal{X}$ (or $\mathcal{X} \times {0,1}$ when including the true label). We have a dataset $D$ of pairs $\{(x_i,y_i)\}_{i=1}^n$ where $y_i \in \{0, 1\}$ is a binary target value for each $x_i \in \mathcal{X}$. Let $h: \mathcal{X} \to \{0, 1\}$ be the binary classifier function. Let $A$ represent a set of protected attributes and $\mathcal{S}$ represent a set of protected subgroups $S \in \mathcal{S}$, represented with indicator functions $S: \mathcal{X} \to \{0,1\}$. Specifically, $\mathcal{S}_A$ represents the set of all intersectional subgroups corresponding to conjunctions over attribute-value pairs from $A$, assuming all attributes $A$ are discrete. Let $\mathcal{S}_L$ be the set of subgroups representable using linear classifiers on protected attributes $A$. We shall refer to $\mathcal{S}_A$ and $\mathcal{S}_L$ as conjunction and linear subgroups, respectively. Finally, let $P(S(x))$ be the probability of some sample $x$ belonging to the subgroup $S \in \mathcal{S}$, given the distribution $\mathcal{D}$, unless specified otherwise.

\section{Background and Related Work}
\label{sec:fair_bg}
The field of fairness evaluation 
\citep{mehrabi_survey_2021} and fair training \citep{hort_bias_2024} is vast. Numerous works define classifier unfairness in different ways, which can generally be split into three categories \citep{mehrabi_survey_2021}: 
\begin{itemize}
    \item Marginal fairness, sometimes referred to as group fairness, suggests that all protected groups should be treated the same. A protected group is an assignment of a single protected feature, e.g., all males (gender = male). 
    \item Individual fairness is achieved when each individual sample is classified the same, regardless of the values of protected attributes. In other words, if two samples differ mainly in protected attributes and are otherwise similar, their classification should be the same. 
    \item Intersectional fairness extends marginal fairness, considering all intersections of marginal groups, e.g., white women older than 60. Like marginal fairness, it looks at (sub)groups of samples. Similar to individual fairness, it considers all protected attributes at once.
\end{itemize}
The field of intersectional fairness can be further divided into multiple notions of fairness \citep{gohar_survey_2023}. We are interested in those notions, where the subgroup may be identified by a description over the set of protected attributes (unlike in Probabilistic Fairness \citep{molina_bounding_2022} or Metric-based fairness \citep{yona_probably_2018}), and which can evaluate bias in data
(unlike Calibration-based fairness \citep{hebert-johnson_multicalibration_2018}, targeting only fairness of the model, given the data).
This is useful when we wish to detect (and describe) groups that are unfairly represented in the data (or unfairly treated by the model).
Note that it is insufficient to use marginal fairness methods by expanding group intersections into new protected attributes, due to the exponential growth of these attributes and the low number of samples per intersection (data sparsity). 

\paragraph{Subgroup Fairness}
In a seminal work, \citet{kearns_preventing_2018} have generalized the concept of Statistical Parity \citep{dwork_fairness_2012} from marginal fairness to intersectional fairness, naming it Statistical Parity Subgroup Fairness (SPSF). 
\begin{definition}[Statistical Parity Subgroup Fairness]
Given a probability distribution $\mathcal{D}$, input space $\mathcal{X}$, a binary classifier $h: \mathcal{X} \to \{0,1\}$, a set of subgroups $\mathcal{S}$, and a parameter $\gamma \in [0, 1]$, we say that $h$ satisfies $\gamma$-SPSF with respect to $\mathcal{S}$ if for every $S \in \mathcal{S}$ 
\begin{equation*}
    P(S(x)) \cdot \left| P(h(x) = 1) - P(h(x) = 1 \mid S(x)) \right| \le \gamma.
\end{equation*}
We will refer to the left-hand side of this inequality as $\mathrm{SPSF}(h, S)$. 
\end{definition}
When applying the multiplication, one can see that the measure essentially evaluates an independence gap, i.e., the difference between the true joint probability and the joint probability under the assumption that the subgroup and the target variable are independent random variables.

In addition, \citet{kearns_preventing_2018} introduced a related notion, called False Positive Subgroup Fairness (FPSF), motivated by the Equality of Opportunity measure \citep{hardt_equality_2016}. 
\begin{definition}[False Positive Subgroup Fairness]
Given a probability distribution $\mathcal{D}$, an input space $\mathcal{X}$, a binary classifier $h: \mathcal{X} \to \{0,1\}$, a set of subgroups $\mathcal{S}$, and a parameter $\gamma \in [0, 1]$, we say that $h$ satisfies $\gamma$-FPSF with respect to $\mathcal{S}$ if for every $S \in \mathcal{S}$ 
\begin{align*}
    P(S(x), y = 0) \cdot \bigl| &P(h(x) = 1 \mid y = 0) - P(h(x) = 1 \mid S(x), y = 0) \bigr| \le \gamma.
\end{align*}
We will refer to the left-hand side of this inequality as $\mathrm{FPSF}(h, S)$. 
\end{definition}

\citet{kearns_preventing_2018} also proposed a learning algorithm based on game theory, where a ``learner'' learns a classifier and an ``auditor'' identifies the subgroup with the greatest unfairness, modifying the learner's task. They use linear models for both the classifier and the auditor's representation of a subgroup.



\paragraph{Ratio-based Fairness Measures}
Another area of fairness measures considers ratios, rather than differences. 
\citet{foulds_intersectional_2020} introduced Differential Fairness, which bounds a logarithm of the ratio $P(h(x) = y \mid S_i(x)) / P(h(x) = y \mid S_j(x))$ by epsilon for all $i,j$ pairs, inspired by differential privacy measures. They mitigated the inaccuracy on small-sample subgroups by using a Dirichlet prior in estimation.
But when there are many protected attributes, the measure becomes uninformative, as its value is determined by the parameter~$\alpha$ \citep{molina_bounding_2022}.  
%
%
%
Additionally, 
Max-min fairness 
evaluates the ratio between the subgroups with the highest and the lowest value of a given performance measure \citep{ghosh_characterizing_2021}. However, this approach requires evaluating all subgroups and suffers from data sparsity.

\paragraph{Unfairness as a distance}
Lastly, there is a direction of viewing unfairness through the lens of distances between distributions. Among others, the positively- and negatively-classified samples can be considered samples from two distributions over the set of protected attributes. If the two distributions are the same, all protected subgroups are equally represented, and the distributional distance will be 0. 
There are works that consider the Wasserstein-1 distance \citep{jiang_wasserstein_2020}, utilize Wasserstein-2 and optimal transport to achieve fairness \citep{silvia_general_2020}, or employ Total Variation \citep{farokhi_optimal_2021}. 

However, all of these distances have exponential sample complexity at best, making them difficult to evaluate reliably. To address that, a sample-efficient distance, designed for intersectional fairness, was recently introduced, called Maximum Subgroup Discrepancy (MSD) \citep{nemecek_bias_2025}:
\begin{definition}[Maximum Subgroup Discrepancy]
\label{def:msd}
Given two probability distributions $\mathcal{D}_{1}$ and $\mathcal{D}_{2}$, defined over a set of discrete protected attributes $A$, let 
\begin{equation*}
    \mathrm{MSD}(\mathcal{D}_{1}, \mathcal{D}_{2}; A) \coloneqq \sup_{S \in \mathcal{S}_{A}} \left| P_{x\sim\mathcal{D}_{1}}(S(x)) - P_{x\sim\mathcal{D}_{2}}(S(x)) \right|, 
\end{equation*}
where $P_{x\sim\mathcal{D}_{1}}(S(x))$ represents the probability of $x$ belonging to a subgroup $S$, given the distribution $\mathcal{D}_{1}$ and the set of subgroups $\mathcal{S}_A$ contains all intersections of attribute-value pairs for all subsets of $A$.
\end{definition}
Its value lies between the Total Variation (or $\ell_1$ distance), which has exponential sample complexity (compared to the linear sample complexity of MSD), and the $\ell_\infty$, which considers only the bottom-level subgroups, leading to issues with data sparsity. Apart from low sample complexity, this measure is well interpretable, returning not only a value but also the subgroup intersection, which is most disproportionately represented between the two distributions \citep{nemecek_bias_2025}.
The authors propose finding the subgroup (and thus evaluating the MSD) via Mixed-Integer Optimization, formulating the problem as finding a conjunction that discriminates between samples from the two distributions.
%


\paragraph{Mixed-Integer Optimization}
Mixed-integer optimization (MIO, also referred to as mixed-integer \emph{programming} \citep{wolsey_integer_2021}) is a powerful framework for modeling and solving mathematical optimization problems, where some decision variables can take only integer values while others are continuously valued. Although MIO, as a general framework for solving NP-hard problems, has been criticized for poor scalability, the computational efficiency of cutting-edge MIO solvers has increased by approximately 22\% annually, \emph{excluding} hardware speedups \citep{koch_progress_2022}. We use the abbreviation MIO, though we consider only mixed-integer \emph{linear} formulations, i.e., those containing only linear constraints.

Due to this steady increase in solver performance, there has been rising interest in MIO from the Machine Learning community, especially from the perspective of interpretable models \citep{justin_responsible_2025}. Indeed, many interpretable models, partly due to their smaller size, have been considered suitable for training via MIO. This includes logical models like DNFs \citep{dash_boolean_2018}, rule lists 
\citep{angelino_learning_2018,rudin_learning_2018},
and even the more structured decision trees \citep{bertsimas_optimal_2017} or their ensembles \citep{misic_optimization_2020}.


\paragraph{Fair Interpretable models}
Additionally, there are many works training fair interpretable classifiers. These
include, for example, rule-based models (FairCORELS \citep{aivodji_faircorels_2021,aivodji_leveraging_2022}, 
DNFs \citep{lawless_interpretable_2023}), decision trees (FDT \citep{aghaei_learning_2019}, DPF \citep{van_der_linden_fair_2022}, FairOCT \citep{jo_learning_2023}) and forests \citep{raff_fair_2018}. Some of these use MIO to train the models; however, all of the above consider only \emph{marginal} fairness, disregarding the need to evaluate intersections, or they naively propose adding intersections as separate groups, which is insufficient due to data sparsity. 

\section{Method}
Our framework has two parts. The first is the detection of the most unfair subgroup, which could be useful for auditing at a given level of fairness, e.g., achieving $\gamma$-SPSF fairness. This requires reliably identifying the subgroup with the highest fairness violation and comparing it to the threshold $\gamma$.

The second part of our framework considers training under intersectional fairness.
Reflecting Rawls's theory of distributive justice \citep{rawls_justice_2001}, we aim to improve fairness for the most disadvantaged groups, using the detection component of our framework to identify such subgroups.

To guarantee finding the \emph{most unfair} group, we will use Mixed-Integer Optimization, a framework that can find global optima. From the methods explored in Section \ref{sec:fair_bg}, we choose SPSF (and FPSF) \citep{kearns_preventing_2018} and MSD \citep{nemecek_bias_2025}. The reasons for this choice are that those measures (1) can evaluate intersectional bias in data, (2) provide a subgroup description, (3) take the data sparsity into account, to be reliably evaluated (this includes having reasonable sample complexity). Finally, (4) they are representable with linear constraints in MIO, unlike ratio-based measures or pattern divergence \citep{pastorLookingTroubleAnalyzing2021}, for example. Importantly, auditing the SPSF measure has been shown not to be polynomially learnable \citep[Thm. 3.3]{kearns_preventing_2018}, thus the use of a generally exponential global optimizer, such as MIO, is justified. 

\subsection{Bias detection}
Firstly, we note that FPSF is almost equivalent to evaluating SPSF on the negative samples. Their relation follows from their definitions:
\begin{equation}
    \mathrm{FPSF}(h, S) = P(y = 0) \cdot \mathrm{SPSF}_{\mathcal{D}_{\mid y = 0}}(h, S), \label{eq:fpsf_spsf}
\end{equation}
where $\mathrm{SPSF}_{\mathcal{D}_{\mid y = 0}}$ is used to denote the SPSF measure on the conditional distribution of negatively labeled samples. Thus, we will focus only on the relationship between SPSF and MSD.
To show that maximizing SPSF and MSD measures lead to detecting the same set of subgroups, we first define Subgroup Discrepancy as the quantity for which MSD attains its supremum.

\begin{definition}[Subgroup Discrepancy]
Given two probability distributions $\mathcal{D}_{1}$ and $\mathcal{D}_{2}$ over the input space $\mathcal{X}$ and a subgroup $S$, let us define Subgroup Discrepancy (SD) as
\begin{equation*}
    \mathrm{SD}(S; \mathcal{D}_{1}, \mathcal{D}_{2}) \coloneqq \left| P_{x\sim\mathcal{D}_{1}}(S(x)) - P_{x\sim\mathcal{D}_{2}}(S(x)) \right|, 
\end{equation*}
where $P_{x\sim\mathcal{D}_{1}}(S(x))$ is the probability of $x$ belonging to a subgroup $S$, given the distribution $\mathcal{D}_{1}$.
\end{definition}
This means that taking a supremum of SD over the set $\mathcal{S}_A$ is the MSD \citep{nemecek_bias_2025}, exactly.
We now use the SD to prove the following theorem:


\begin{theorem}[SPSF and SD relation]
\label{thm:proport}
Given a probability distribution $\mathcal{D}$ over an input space $\mathcal{X}$, a binary classifier $h: \mathcal{X} \to \{0, 1\}$, and a subgroup $S$, the Statistical Parity Subgroup Fairness (SPSF) metric \citep{kearns_preventing_2018} and the Subgroup Discrepancy (SD) satisfy the following relation:
\[
\mathrm{SD}(S; \mathcal{D}_{\mid h(x)=1}, \mathcal{D}_{\mid h(x)=0}) = \frac{\mathrm{SPSF}(h, S)}{P(h(x) = 1) \cdot P(h(x) = 0)}, 
\]
where $\mathcal{D}_{\mid h(x)=1}, \mathcal{D}_{\mid h(x)=0}$ are the conditional distributions of positively- and negatively-classified samples, respectively.
\end{theorem}

\begin{proof}
We start with the definition of Subgroup Discrepancy and show that it equals a multiple of SPSF. For visual clarity, let us simplify the notation by denoting $P(S(x))$ as $P(S)$, $P(h(x) = 1)$ as $P(h)$, and $P(h(x) = 0)$ as $P(\Bar{h})$. Then
\begin{align*}
    \mathrm{SD}(S; \mathcal{D}_{\mid h(x)=1}, \mathcal{D}_{\mid h(x)=0}) &= \left| P(S \mid h) - P(S \mid \bar{h}) \right|, \\
    \intertext{then utilizing the Bayes' rule}
    &= \frac{1}{P(h) \cdot P(\Bar{h})} \cdot \left|P(h)P(S, \Bar{h}) - P(\Bar{h})P(S, h)\right|; \\
    \intertext{denoting the $1 / (P(h) \cdot P(\Bar{h}))$ as $p_h$ for visual clarity and using the fact that $P(h) = 1 - P(\bar{h})$ we get}
    &= p_h \cdot \left|P(h)P(S, \Bar{h}) - \left(1 - P(h)\right)P(S, h)\right|\\
    &= p_h \cdot \left|P(h)\left(P(S, \Bar{h}) + P(S, h)\right) - P(S, h)\right|, \\
    \intertext{and since the sum of the two joint probabilities is equal to the marginal probability}
    &= p_h \cdot \left|P(S)P(h) - P(S, h)\right| \\
    &= p_h \cdot P(S)\left|P(h) - P(h \mid S)\right|, 
\end{align*}
which, after expanding the $p_h$,  corresponds exactly to the sought relation. 
\end{proof}

Clearly, this relation assumes a non-constant classifier, which would lead to a 0 in the denominator. The equivalence of MSD and maximizing SPSF in detecting the most unfair subgroup follows from this proportional relation between SPSF and SD, due to the independence of the denominator on the subgroup $S$.

\begin{corollary}[Equivalence of SPSF and MSD]
It follows from the relation in Theorem \ref{thm:proport} that the set of subgroups $S \in \mathcal{S}_A$ that maximize the SPSF metric is identical to the set of subgroups achieving the Maximum Subgroup Discrepancy between the distributions of positively- and negatively-classified samples:
\[
\arg\sup_{S \in \mathcal{S}_A} \mathrm{SPSF}(h, S) = \arg\sup_{S \in \mathcal{S}_A} \mathrm{SD}(S; \mathcal{D}_{\mid h(x)=1}, \mathcal{D}_{\mid h(x)=0}),
\]
where the right side of the equation is the set of subgroups achieving $\mathrm{MSD}(\mathcal{D}_{\mid h(x)=1}, \mathcal{D}_{\mid h(x)=0}; A)$.
This holds for any set of subgroups $\mathcal{S}$, but we use the $\mathcal{S}_A$ in line with the definition of MSD (Def. \ref{def:msd}). 
\end{corollary}

Please note that the two measures are equivalent only in terms of detecting the subgroup; their values can differ significantly, especially for classifiers with imbalanced classes. This can influence the training under a given measure. We can, however, say that $h$ satisfies $\gamma$-SPSF fairness with respect to $\mathcal{S}_A$ by evaluating MSD on the conditional distributions of positively- and negatively-classified samples, and finding $$\mathrm{MSD}(\mathcal{D}_{\mid h(x)=1}, \mathcal{D}_{\mid h(x)=0}; A) \le \frac{\gamma}{P(h(x) = 1) \cdot P(h(x) = 0)}.$$
A similar relation holds for $\gamma$-FPSF fairness, by expressing FPSF in terms of SPSF (see Eq. \ref{eq:fpsf_spsf}).

\paragraph{Choosing the subgroup definition}
In terms of the set of possible subgroups, \citet{nemecek_bias_2025} used the $\mathcal{S}_A$ class, represented by conjunctions over the (discrete) protected attribute-value pairs, while \citet{kearns_preventing_2018} proposed working with a more general set of subgroups represented by a linear classifier ($\mathcal{S}_L$).

Despite the need for discretization, we primarily use conjunctions because, in detection, we wish to identify the single most disadvantaged subgroup rather than a union of multiple (partial) subgroups, as would be possible with linear models. This makes the subgroups clearer, which can be useful for designing specialized treatment for a given subgroup in bias mitigation, for example \citep{sharma_facilitating_2024}. Furthermore, because $\mathcal{S}_A \subseteq \mathcal{S}_L$, it should be less constraining and allow for classifiers with higher performance while guaranteeing the same level of fairness for every subgroup. 
Discretization, on the other hand, is not necessarily a problem in bias evaluation, since most protected attributes are discrete, or can be discretized meaningfully (e.g., $\mathrm{age} > 40$ in employment use cases in the US \citep{adea_age_1967}). 

\subsection{Training Framework} 
Generally, we aim to optimize the following problem
\begin{subequations}
\label{eq:train_form}
\begin{align}
    \min \; & \sum_{i=1}^n \llbracket \hat{y}_i \ne y_i \rrbracket/N_{y_i} \label{eq:acc}\\ 
    \mathrm{s. t.} \; & h(x_i) = \hat{y}_i & \forall i \in \{1, \ldots, n\}  \label{eq:classif}\\
    & \mathrm{Unfairness}_S(x, y, \hat{y}) \le \gamma & \forall S \in \mathcal{S} \label{eq:fair}
\end{align}
\end{subequations}
where $N_{y_i}$ is the number of samples with class $y_i$, the $\hat{y_i} \in \{0,1\}$ is a variable representing the estimated class of sample $x_i$, \eqref{eq:acc} is a balanced 0-1 loss we are minimizing, \eqref{eq:classif} represents the constraints of the classifier we are training, and $\mathrm{Unfairness}_S$ is some measure of intersectional bias over the subgroup $S$, in our case an empirical estimate of FPSF$(h, S)$, SPSF$(h, S)$, or SD$(S; \mathcal{D}_{\mid h(x)=1}, \mathcal{D}_{\mid h(x)=0})$ on dataset $D$. Each measure can be represented using a single linear constraint, although the Subgroup Discrepancy requires adding $\mathrm{O}(n)$ variables and constraints to the base model. Details are in Appendix~\ref{app:MIO}. Note that this is the case only for training; in detection, the measures differ in the objective function (and number of used samples - constraints, in the case of FPSF).

Expressing the fairness constraint \eqref{eq:fair} for all subgroups would mean including an exponential number of constraints. However, only a few of the most violating subgroups can be enough in practice, since the rest will score below $\gamma$. In MIO, we shall, therefore, represent the set of constraints \eqref{eq:fair} with lazy constraints \citep{pearce_towards_2019}, also known as cutting planes. This means that when a MIO solver finds a feasible solution, we will find the most unfair subgroup given the classifier, and if the found subgroup violates the threshold $\gamma$, we add the subgroup as a constraint to the model (we ``cut off'' the incumbent solution) and re-solve it. 
This way, we only need to add a few constraints (often fewer than 10; see Table \ref{tab:data}) rather than an exponential number. 
To speed up the search, we can constrain the objective value to be above the $\gamma$ threshold, or even remove the maximization objective to find any feasible solution above the $\gamma$ threshold, which was the approach taken in experiments.


Regarding the types of classifiers $h(\cdot)$ that can be optimized, many ML models can be formulated and trained using MIO (some to provable optimality) with linear constraints. These include many classifiers considered interpretable: classification trees \citep{bertsimas_optimal_2017}, rule sets \citep{dash_boolean_2018}, decision lists \citep{rudin_learning_2018}, ensembles \citep{misic_optimization_2020}, and more. 
Indeed, interpretable models usually have fewer parameters than black-box models like neural networks, making them well-suited for MIO, which has limited scalability. Additionally, MIO enables finding globally optimal solutions, i.e., models with the highest performance, which is even more crucial, given the inherent interpretability-performance tradeoff.


Our framework allows the user to choose any MIO-representable ML model and train it under intersectional fairness constraints. 
The detection component can be used separately to identify the most disadvantaged subgroup, either for auditing or for targeted bias mitigation strategies (e.g., \citep{sharma_facilitating_2024}).
The used formulations are based on existing ones; see details in Appendix \ref{app:MIO}.



\section{Experiments}
We empirically validate the relationship between the bias measures, explore the use of MIO for bias detection, and showcase the training framework. We perform experiments on five real-world datasets from the US Census using the folktables library \citep{ding_retiring_2021} (under the MIT license, built on the public Census Bureau Data API). We consider the 1-year horizon survey of 2018 for the state of California. Each configuration is run on five random subsamples of the data. See dataset summary in Table \ref{tab:data}.

\begin{table}[t]
\caption{
The number of samples, protected attributes ($|A|$), and intersectional subgroups representable by conjunctions of attribute-value pairs. And the mean number of subgroups (i.e., cuts of a candidate solution) considered during training of the fair models. Values after $\pm$ represent standard deviation.
}
\label{tab:data}
\centering
\resizebox{\linewidth}{!}{
\begin{tabular}{lrrrrrr}
    \toprule
         Dataset name & \# samples & $|A|$ & \# subgroups ($\mathcal{S}_A$) & Conjunction cuts & Linear cuts & Overall cuts \\
         \midrule
         ACS Income  & 195,665 & 4 & 1,229 &  3.80 $\pm$ 1.03 &  11.33 $\pm$ 10.13 &     8.32 $\pm$ 8.63 \\
         ACS Public Coverage  & 138,554 & 12 & 56,008,799 &  9.90 $\pm$ 1.79 &  13.93 $\pm$ 11.14 &    12.32 $\pm$ 8.81 \\
         ACS Mobility  & 80,329 & 11 & 11,567,699 &   0.80 $\pm$ 0.79 &   21.33 $\pm$ 17.00 & 13.12 $\pm$ 16.56 \\
         ACS Employment  & 378,817 & 11 & 18,381,599 &  22.50 $\pm$ 8.59 &  15.87 $\pm$ 13.55 &  18.52 $\pm$ 12.07 \\
         ACS Travel Time  & 172,508 & 6 & 47,339 &   5.00 $\pm$ 1.56 &   10.67 $\pm$ 9.80 &    8.40 $\pm$ 8.06 \\
\bottomrule
\end{tabular}
}
\end{table}

We utilize the open-source Pyomo \citep{bynum_pyomo_2021} as an MIO modeling library. 
We use the Gurobi solver \citep{gurobi} to solve the formulations, i.e., train the models. 
The algorithms were run on an internal cluster with sufficient RAM and an Intel Xeon Scalable Gold 6146 or AMD EPYC 7543 processor 
with 16 cores available. All code is in the supplementary material and will be open sourced under CC-BY license.

\subsection{Detecting the most unfair subgroup}
\label{sec:eval_results}
\begin{figure}[b]
    \centering
    \includegraphics[width=\linewidth]{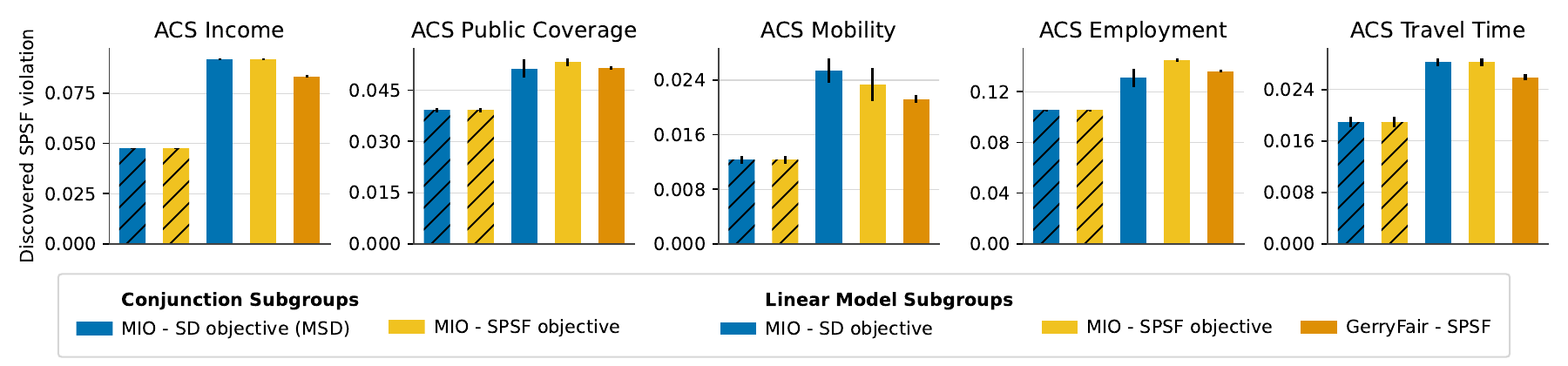}
    \caption{
    Mean fairness violation (and standard deviations). In detection, finding the subgroup with more unfairness is better. The bars with slanted lines represent the optimal solutions found over conjunction subgroups ($\mathcal{S}_A$), which are a subset of linear subgroups ($\mathcal{S}_L$), explaining the difference to the other setups. The blue bar with slanted lines to the left corresponds to MSD, and the orange bar to the right corresponds to GerryFair; the rest is our work.
    }
    \label{fig:eval}
\end{figure}


To compare the SPSF and MSD in a practical setting, we use the true label $y$ as the vector of predictions, evaluating bias in the data itself (i.e., the bias of a perfect classifier $h(x_i) = y_i$). To evaluate consistency, we subsample 50,000 points and find subgroups maximizing SPSF and SD violation, using the MIO approach (amounting to MSD in the case of maximizing SD on conjunction-based subgroups $\mathcal{S}_A$) and the GerryFair \citep{kearns_preventing_2018}, which uses linear threshold functions as subgroups and maximizes SPSF. To better compare to GerryFair, we also use MIO to find the most unfair subgroup represented by a linear classifier ($\mathcal{S}_L$). The linear groups are warm-started with a solution from either GerryFair (if it was numerically stable for MIO) or a logistic regression model. The solver was run with a time limit of 10,000 seconds.

The results, see Figure \ref{fig:eval}, empirically validate our theoretical result. The conjunction-based subgroups, which are proven optimal, are the same for SD and SPSF. With linear subgroups MIO approach consistently finds subgroups with notably higher unfairness than GerryFair. The use of SD and SPSF as objectives yields different results here because the solver does not find the optimal solution within the given time. Note that MIO can dominate GerryFair even when run for just 5 minutes; see Appendix \ref{app:time} for details. These results suggest that GerryFair might claim $\gamma$-SPSF fairness in cases when the opposite can be shown to be true.




\subsection{Training framework}
\label{sec:train_results}

\begin{figure}[t]
    \centering
    \includegraphics[width=\linewidth]{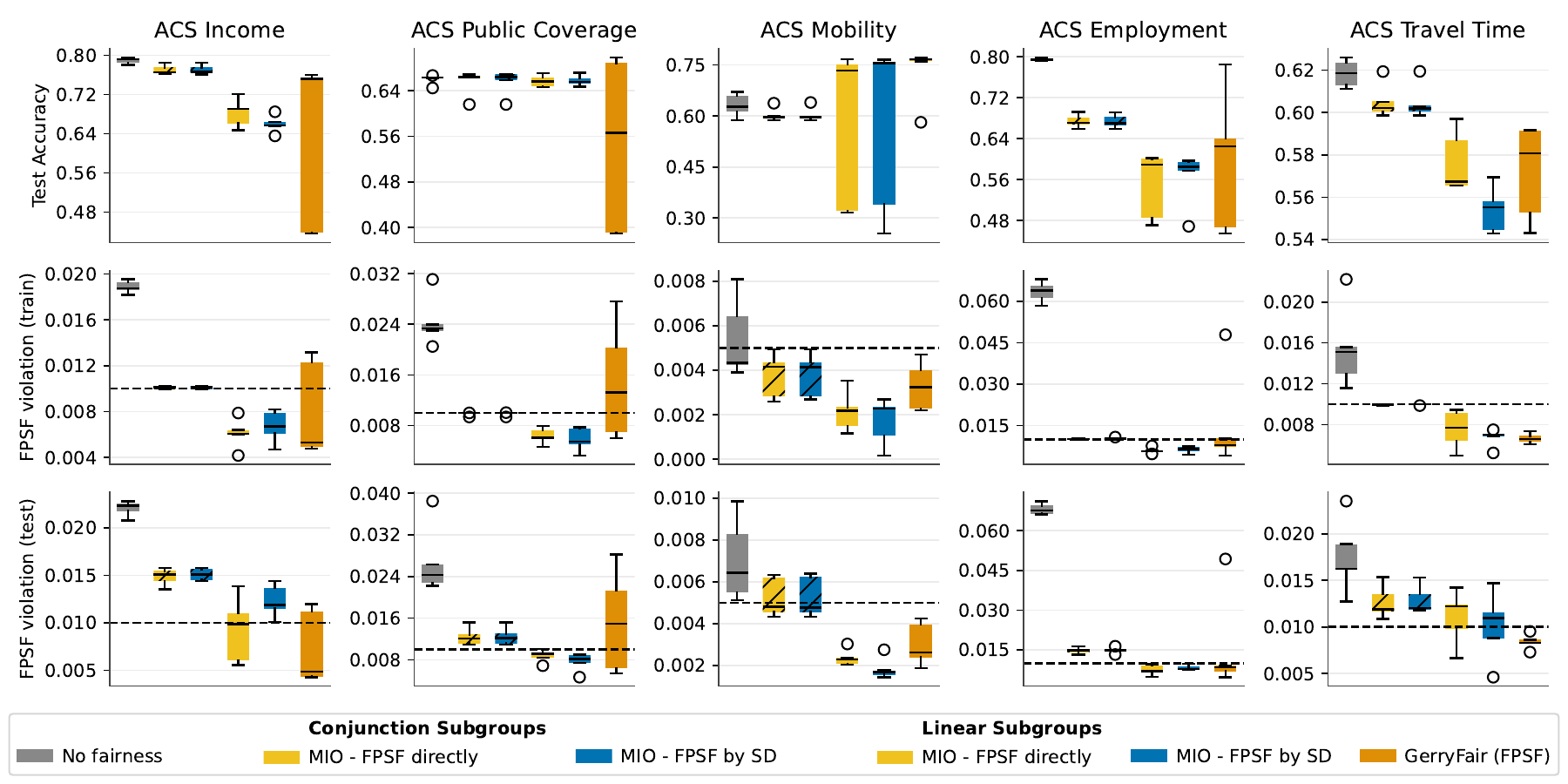}
    \caption{
    Comparison of trained classifiers. We compare our MIO-based approach to a classifier without fairness constraints and GerryFair. The dashed line represents the fairness threshold $\gamma$ used during training. Blue boxes indicate using SD as a proxy for FPSF to identify the violating subgroup; in training, all use the FPSF measure. We evaluate FPSF violation using MIO on conjunction subgroups ($\mathcal{S}_A$).
    }
    \label{fig:train}
\end{figure}

In training, we focus on the FPSF measure in training linear models, as in GerryFair \citep{kearns_preventing_2018}. 
 We compare a classifier trained with MIO without any fairness constraints to 4 classifiers with fairness constraints: 2 trained with FPSF constraints using logical conjunctions as subgroups and 2 with linear subgroups, to better compare with GerryFair. All are trained using the FPSF constraints, but two (one with conjunction subgroups and one with linear) use SD (on samples from the negative class) to search for the violating subgroup. We solve the formulations with a 4-hour time limit (with 5 minutes to find a violating subgroup) and limit the number of GerryFair iterations to 2880 (5 seconds per iteration, to match the 4-hour limit, though they often took almost twice that). We take 20,000 random samples from each dataset, halving them to training and test sets. We limit the time to find a violating subgroup (the lazy constraint) to 5 minutes and set the fairness threshold $\gamma$ to 0.01 in all datasets, except for ACS Mobility, where we set it to 0.005, because even the baseline model satisfied 0.01-FPSF fairness.


As expected, using the global optimizer and the less restrictive (but more natural) conjunction subgroups ($\mathcal{S}_A$) instead of linear models subgroups ($\mathcal{S}_L$) improves classifier performance, as evidenced by the accuracy results in Figure \ref{fig:train}. Our framework compares favorably to the model that disregards fairness, suffering only a modest drop in accuracy (even negligible in some cases). Further, given the MIO's strict adherence to constraints, the fairness threshold on the training data is kept almost exactly (up to numerical precision), as shown in the second row of Figure \ref{fig:train}. On out-of-sample data (third row), the FPSF violation of our framework is slightly above the threshold.
Nevertheless, the violation is mostly below 1.5x the threshold and notably lower than that of the base model.

Furthermore, when considering the wider definition of subgroups ($\mathcal{S}_L$), we see that $\gamma$-FPSF fairness is usually satisfied even on the test data. It often achieves performance comparable to or better than GerryFair while constraining the unfairness more reliably. Specifically, notice the ACS Public Coverage results, where GerryFair underperforms in Accuracy and also fails to adhere to the fairness constraint (using a narrower subgroup definition) even on the training data.

Overall, the performance of FPSF or SD in finding the violating subgroup seems equivalent for conjunction subgroups. This is not so much the case for the linear subgroups, where SD performs slightly worse on ACS Income and ACS Travel Time. This suggests a clearer dominance of the FPSF objective, compared to the results in bias detection in Figure \ref{fig:eval} (recall that FPSF is equivalent to SPSF on negative samples, as per Eq. \ref{eq:fpsf_spsf}). While those suggest that SPSF and SD each have their use in evaluation, using SPSF (FPSF) works better in training more generally.

\paragraph{Lazy cuts}
Importantly, the lazy approach to fairness constraints makes the MIO approach feasible. It achieves $\gamma$-FPSF fairness while constraining a few dozen subgroups (at most 47), instead of the millions that exist (or even more for the $\mathcal{S}_L$ class). See Table \ref{tab:data} for the mean counts of cuts, i.e., subgroups used in constraints. When working with linear subgroups, more cuts are generally needed, which is understandable, given that the class of subgroups is bigger. There is one exception to this: the ACS Employment dataset, where the nature of the problem makes linear cuts more effective at mitigating bias, so that fewer cuts cover more bias.
\begin{figure}
    \centering
    \includegraphics[width=\linewidth]{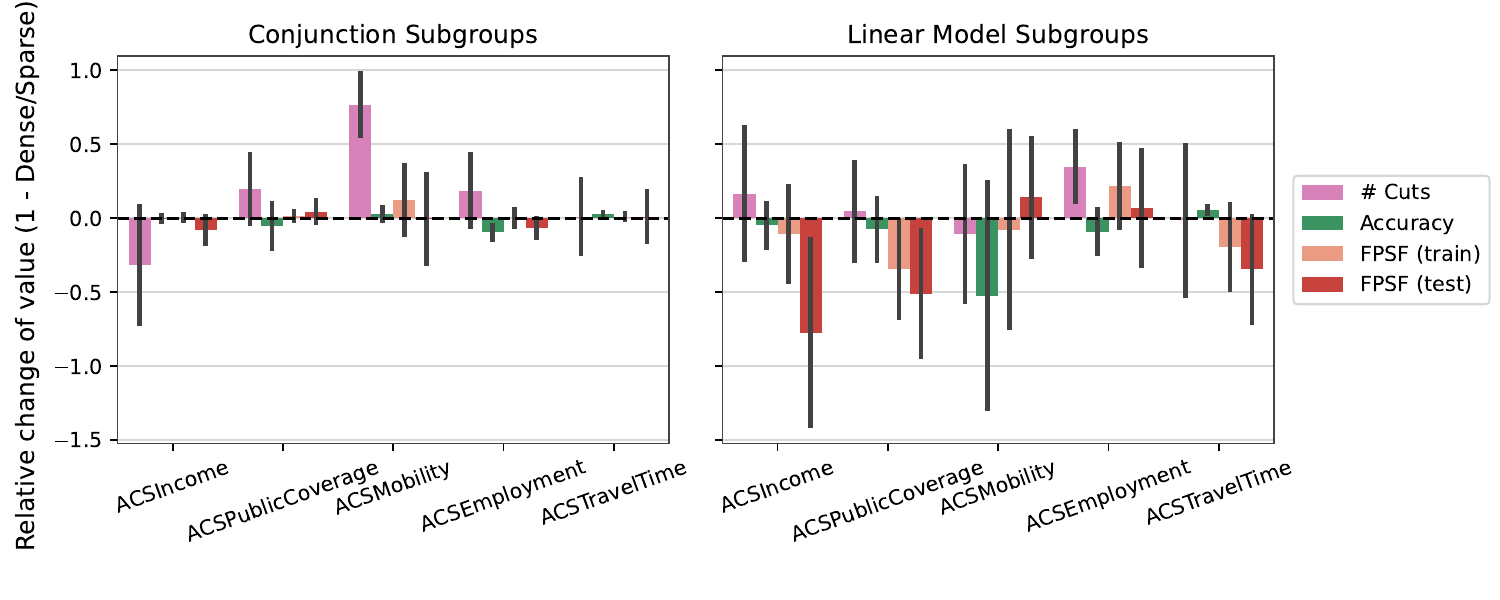}
    \caption{Change in fairness, accuracy, and number of cuts, after introducing a sparsity term into the objective. A positive value means an increase in value after introducing the sparsity term.}
    \label{fig:sparsity}
\end{figure}

\paragraph{Sparsity effects}
Since we aim to train interpretable models, we further investigate the effect of introducing a sparsity term to the optimization. We compute sparsity as the proportion of non-zero coefficients of the linear model and add it as a penalty to the objective, with a multiplicative factor of 0.1. We find that this zeroes 54\% (or 76\% when considering linear subgroups) of coefficients on average, compared to almost 0\% when disregarding sparsity. Importantly, in the sparse setting, we mostly find classifiers with accuracy comparable to that in the dense setting.

See Figure \ref{fig:sparsity}, where we aggregate runs using SD and FPSF together and show the relative change in accuracy, fairness, and number of cuts after introducing the sparsity term. Interestingly, in linear subgroups, introducing sparsity often leads to a notable decrease in unfairness on the test set, while the number of cuts (subgroups) generated does not change drastically. The increase in ACS Mobility can be explained by the low base number of cuts (0.8 on average) in the dense setting.

\section*{Conclusion}
We have shown that MSD and SPSF yield the same set of subgroups when detecting the most unfair subgroup. Since MSD amounts to optimizing balanced 0-1 error, this might open the possibility of using more general training algorithms even for SPSF. We have shown that MIO identifies subgroups with higher bias, even the guaranteed most biased, an essential task in auditing or certification.
%
We have demonstrated that the lazy approach to fairness constraint generation 
enables consideration of exponentially many fairness constraints 
within the framework of intersectionally fair training. 
The detection part can be used similarly in other optimization techniques, including gradient-based ones. 

\bibliographystyle{plainnat}
\bibliography{references}


\appendix



\section{Ethical Statement and Broader Implications}
While our framework optimizes for the broader intersectional fairness, we acknowledge that this is merely a technical proxy for true societal justice. Further, our fairness guarantee applies only to training, while true bias mitigation requires continuous oversight. 

We also acknowledge the potential dual use of bias detection, with malicious users using it to target already disadvantaged subgroups. The use of MIO also incurs a higher computational cost, which we believe is justified by the improved quality of the results.

When mitigating bias, one should carefully investigate the entire ML pipeline.
Further, despite our consideration of intersectional fairness, there remain many difficulties in achieving true bias mitigation \citep{alvarez_policy_2024}.
Despite the use of intersectional fairness, it might happen that some correlated protected attributes are not observed, e.g., caregiving influencing the gender pay gap. In such cases equalizing might lead to disadvantaging male caregivers. This is referred to as the Debiasing paradox \citep{smirnovQuotabasedDebiasingCan2021} and should be considered when implementing the proposed method in practice.

\paragraph{LLM use}
Large Language Models were used to improve text flow, for grammar checks, and to generate code for visualizations. All uses were thoroughly checked by the authors, and they take full responsibility. Earlier versions of the paper were also subject to LLM-based review. 

\section{MIO formulations}
\label{app:MIO}
We present the MIO formulations used in the framework in their entirety. We start with the formulations for bias detection, and then show the training formulations. 

\subsection{Detection formulation}
The formulation for bias detection using conjunction groups, modified from MSD \citep{nemecek_bias_2025}, reads as follows:
\begin{subequations}
\label{eq:conjunct}
\begin{align}
    \max \; & \mathrm{Unfairness}(\hat{y}) \\
    \mathrm{s.t.} \; & \hat{y_i} \le 1 - (u_j - x_{i,j} \cdot u_j) & i \in D, \; j \in A \label{eq:mio_positive} \\
    & \hat{y_i} \ge 1 - \sum_{j \in A}(u_j - x_{i,j} u_j) & i \in D \label{eq:mio_negative} \\
    & \sum_{i \in D} \hat{y_i} \ge N_{\mathrm{min}} &  
    \label{eq:mio_lowerbound} \\
    & 0 \le \hat{y_i} \le 1 &  i \in D \label{eq:mio_yhatbounds}\\
    & u_j \in \{0, 1\} & j \in A,
\end{align}
\end{subequations}
where we assume $x_{i,j}$ are binary parameters, each corresponding to $j$-th feature of $i$-th sample from a dataset $D$, binary $u_j$ is 1 when feature $j$ is present in the conjunction and $\hat{y}_i$ is equal 1 if the $i$-th sample is within the subgroup. An optional parameter is $N_\mathrm{min}$ represents the size of a minimal subgroup to be considered. We keep this value 0, effectively omitting the constraint \eqref{eq:mio_lowerbound}. Because all $x$ are binary, the variables $\hat{y}_i$ can be real-valued, notably speeding up the solution process.

\paragraph{SD objective}
As the objective function $\mathrm{Unfairness}(\hat{y})$, we would like to put SPSF, FPSF, or (M)SD. As shown in the paper, they are all related by factors constant with respect to this formulation. However, the scale of the coefficients can affect the solution process, which is why we test each of them. The MSD objective is:
\begin{subequations}
\label{eq:MSD_opt}
    \begin{align}   
    \max \; & o \\
    & o \le \frac{1}{|D_+|} \sum_{i \in D_+}{\hat{y_i}} - \frac{1}{|D_-|} \sum_{i \in D_-}{\hat{y_i}} + 2b \label{eq:mio_abs1} \\
    & o \le \frac{1}{|D_-|} \sum_{i \in D_-}{\hat{y_i}} - \frac{1}{|D_+|} \sum_{i \in D_+}{\hat{y_i}} + 2(1-b) \label{eq:mio_abs2} \\
    &b \in \{0, 1\},
    \end{align}
\end{subequations}
where $D_+$ and $D_-$ correspond to positively- and negatively-classified subsets of the data, respectively. constraints \eqref{eq:mio_abs1} and \eqref{eq:mio_abs2} bound the value o from above so that $o$ corresponds to the absolute value of the difference, as per the (M)SD definition.

\paragraph{SPSF objective}
For SPSF, this would look like:
\begin{subequations}
\label{eq:spsf_opt}
    \begin{align}   
    \max \; & o \\
    & o \le \frac{|D_-|}{|D|^2} \sum_{i \in D_+}{\hat{y_i}} - \frac{|D_+|}{|D|^2} \sum_{i \in D_-}{\hat{y_i}} + 2b \\
    & o \le \frac{|D_+|}{|D|^2} \sum_{i \in D_-}{\hat{y_i}} - \frac{|D_-|}{|D|^2} \sum_{i \in D_+}{\hat{y_i}} + 2(1-b) \\
    &b \in \{0, 1\},
    \end{align}
\end{subequations}
where the difference is only in the multiplicative coefficients. The factor 2 is no longer a tight big-M bound, but a tight one may be computed as $$\frac{2 \cdot |D_+| \cdot |D_-|}{|D|^2},$$
which is guaranteed to be less than 2, making 2 a safe option for a big-M constant.

\paragraph{FPSF objective}
We do not directly optimize the FPSF as an objective, since the scale of the coefficients does not differ greatly from that in SPSF. Nevertheless, its formulation could look like:
\begin{subequations}
\label{eq:fpsf_opt}
    \begin{align}   
    \max \; & o \\
    & o \le \frac{|D_{0,-}|}{|D_0| \cdot |D|} \sum_{i \in D_{0,+}}{\hat{y_i}} - \frac{|D_{0,+}|}{|D_0| \cdot |D|} \sum_{i \in D_{0,-}}{\hat{y_i}} + 2b  \\
    & o \le \frac{|D_{0,+}|}{|D_0| \cdot |D|} \sum_{i \in D_{0,-}}{\hat{y_i}} - \frac{|D_{0,-}|}{|D_0| \cdot |D|} \sum_{i \in D_{0,+}}{\hat{y_i}} + 2(1-b) \\
    &b \in \{0, 1\},
    \end{align}
\end{subequations}
where the subscript $0$ means that we consider samples for which the true class $y$ equals 0. A tighter big-M bound would be computed similarly to SPSF $$\frac{2 \cdot |D_{0,+}| \cdot |D_{0,-}|}{|D| \cdot |D_0|}.$$

The complete formulation thus consists of choosing an objective out of the formulations \eqref{eq:MSD_opt}, \eqref{eq:spsf_opt}, or \eqref{eq:fpsf_opt} and adding it to the main formulation for finding conjunctions \eqref{eq:conjunct}.

\paragraph{Linear subgroups}
Alternatively, we have also used a Linear representation of a subgroup. In that case, the main body of the formulation could look like: 
\begin{subequations}
\label{eq:linear}
\begin{align}
    \max \; & \mathrm{Unfairness}(\hat{y}) \\
    \mathrm{s.t.} \; & c^\top x_i - t \ge (\hat{y_i} - 1) \cdot 2|A| & i \in D \label{eq:lin0} \\
    & c^\top x_i - t \le \hat{y_i} \cdot (2|A| + \epsilon) - \epsilon & i \in D \label{eq:lin1} \\
    & \sum_{i \in D} \hat{y_i} \ge N_{\mathrm{min}} &  \\
    & -1 \le c_j \le 1 &  j \in A \\
    & -|A| \le t \le |A|  \\
    & \hat{y}_i \in \{0, 1\} & i \in D,
\end{align}
\end{subequations}
where $c$ is a vector of coefficients, $t$ is the decision threshold, $x$ are assumed normalized to $[0,1]$ range and $\epsilon$ is a small constant ($10^{-6}$ in our experiments) enforcing strictness of the constraint \eqref{eq:lin1}. The rest is the same as in \eqref{eq:conjunct}.

This formulation differs from the conjunction-based formulation in a couple of ways. Firstly, the number of discrete (binary) variables is $O(n)$ where $n$ is the number of samples, while in the conjunction formulation, this was $O(|A|)$, i.e., linear in the number of (binarized) features. Due to binarization, the conjunction-based formulation can also more efficiently reduce the overall number of samples by summing up the coefficients of the same vectors $x$. This can be done (and we do it) even for linear subgroups, but the gains are not as large when continuous features are involved. 

\paragraph{Fairness threshold}
Additionally, we have discussed adding a constraint on the detected fairness to improve solving time. If we utilize the above formulation for bias auditing and we know what $\gamma$ we wish to satisfy, we can add a simple constraint $o \ge \gamma$ to the formulation. Optionally, we can remove the objective and treat it as a feasibility problem. When we are using (M)SD in place of FPSF, we consider only the part of the data where $y=0$ and multiply the gamma in this constraint accordingly, leading to 
$$o \ge \frac{\gamma}{P(y = 0)P(h(x) = 1 \mid y = 0)P(h(x)= 0 \mid y = 0)}.$$

\subsection{Training formulation}
To describe the training formulation, we restate the formulation skeleton from \eqref{eq:train_form}:
\begin{subequations}
\begin{align}
    \min \; & \sum_{i=1}^n \llbracket \hat{y}_i \ne y_i \rrbracket/N_{y_i} \\ 
    \mathrm{s. t.} \; & h(x_i) = \hat{y}_i & \forall i \in \{1, \ldots, n\} \label{eq:h} \\
    & \mathrm{Unfairness}_S(x, y, \hat{y}) \le \gamma & \forall S \in \mathcal{S}. \label{eq:subg_fair} 
\end{align}
\end{subequations}
We will examine the model constraints \eqref{eq:h} and fairness constraints \eqref{eq:subg_fair} separately.

Starting with the predictor representation, we can choose from many types, as suggested in the main body. We implement two types: a linear classifier and a DNF (i.e., a rule set). 

\paragraph{Linear classifier} 
A linear classifier is practically the same as the linear model representation of the subgroup, except for the objective, which also includes the optional sparsity term
\begin{subequations}
\begin{align}
    \min \; & \frac{1}{|D_0|}\sum_{i \in D_0}\hat{y}_i + \frac{1}{|D_1|}\sum_{ \in D_1}(1 - \hat{y}_i) + \frac{\sigma}{d} \sum_{j=1}^d s_j \\
    \mathrm{s.t.} \; & c^\top x_i - t \ge (\hat{y_i} - 1) \cdot 2d & i \in D \\
    & c^\top x_i - t \le \hat{y_i} \cdot (2d + \epsilon) - \epsilon & i \in D \\
    & s_j \ge c_j  & j \in \{1,\ldots,d\} \label{eq:sparsity1}\\
    & s_j \ge -c_j & j \in \{1,\ldots,d\}  \label{eq:sparsity2} \\
    & -1 \le c_j \le 1 & j \in \{1,\ldots,d\} \\
    & -d \le t \le d  \\
    & \hat{y}_i \in \{0, 1\} & i \in D \\
    & s_j \in \{0, 1\} & j \in \{1,\ldots,d\},
\end{align}
\end{subequations}
and the necessary constraints \eqref{eq:sparsity1} and \eqref{eq:sparsity2}, such that $s_j$ is equal to 1 when the coefficient $c_j$ is not 0. In experiments with a sparse linear model, we use the sparsity weight parameter $\sigma = 0.1$. $D_0$ and $D_1$ are sets of samples with $y=0$ and $y=1$, respectively. 
Also note that here, $d$ is the total number of features, and $x_i$ is the full sample $i$, rather than just its protected features.

\paragraph{DNF formulation}
Inspired by the work of \citet{su_learning_2016}, we use the following formulation to represent a DNF with $C$ terms via a representation of a CNF:
\begin{subequations}
\label{eq:dnf}
\begin{align}
    \min \; & \frac{1}{|D_0|}\sum_{i \in D_0}e_i + \frac{1}{|D_1|}\sum_{i \in D_1}e_i \\
    \mathrm{s.t.} \; & c_{i,k} \ge \bar{x}_{i,j} u_{j,k} & i \in D, j \in \{1,\ldots,d\}, k \in \{1,\ldots,C\} \\
    & c_{i,k} \le \sum_{j=1}^d \bar{x}_{i,j} u_{j,k} & i \in D,  k \in \{1,\ldots,C\} \\
    & e_i \ge 1 - \frac{1}{C} \sum_{k=1}^C c_{i,k} & i \in D_0 \\
    & e_i \le C - \sum_{k=1}^C c_{i,k} & i \in D_0 \\
    & e_i \ge \left( \sum_{k=1}^C c_{i,k} \right) - (C - 1) &  i \in D_1 \\
    & e_i \le \frac{1}{C} \sum_{k=1}^C c_{i,k}  &  i \in D_1 \\
    & 0 \le c_{i,k} \le 1 & i \in D, k \in \{1,\ldots,C\} \\
    & e_i \in \{0,1\} & i \in D \\
    & u_{j,k} \in \{0, 1\} & j \in \{1,\ldots,d\}, k \in \{1,\ldots,C\},
\end{align}
\end{subequations}
where $e_i$ is a variable representing the 0-1 error of $i$-th sample, $u_{j,k}$ is 1 if $j$-th feature is used in $k$-th clause, and $c_{i,k}$ is 1 if $i$-th sample violates the DNF term, corresponding to satisfying the $k$-th clause. The bar above $\bar{x}_{i,j}$ means that we consider the negation of the input, in order for this CNF formulation (with flipped $y$ as well) to correspond to a DNF, using De Morgan's laws. Compared to \citet{su_learning_2016}, we have to consider bounds on $e_i$ and $c_{i,k}$ from both sides, to assign the error tightly, in order for them to have the correct value, such that the fairness constraints are satisfied. 
We can easily recover $\hat{y}_i$ by setting $\hat{y}_i = e_i$ for negative samples and $\hat{y}_i = 1 - e_i$ for positive ones.

\paragraph{The unfairness constraint}
For a given subgroup $S$, represented here as a set of sample indices, we now define the unfairness constraint. These constraints will be added for each cut. Firstly, we determine the ``direction'' of the violation. This means taking the sign of the given measure before taking the absolute value. This will allow us to add a single constraint, rather than adding two constraints, one for each possible relative value.

We can formulate the SPSF straightforwardly:
\begin{equation}
\frac{|S|}{|D|^2}\sum_{i \in D} \hat{y}_i - \frac{1}{|D|}\sum_{i \in S} \hat{y}_i \le \gamma.
\end{equation}
In case the relative violation before taking the absolute value was positive. Otherwise, we would negate the left-hand side.

Similarly for FPSF:
\begin{equation}
\frac{|S_0|}{|D|\cdot |D_0|}\sum_{i \in D_0} \hat{y}_i - \frac{1}{|D|}\sum_{i \in S_0} \hat{y}_i \le \gamma,
s\end{equation}
where $S_0$ is the set of subgroup indices for which the true $y=0$, similar to the $D_0$.

\paragraph{SD constraint}
The most complicated to put into a constraint is the SD measure. While SPSF and FPSF use only the $\hat{y}_i$ value, SD requires additional variables due to its conditioning on the classification outcome. 
We formulate it in the following way:
\begin{equation}
    \sum_{i \in S} p_i^+ - \sum_{i \in S} p_i^- \le \gamma,
\end{equation}
where $p_i^+$ and $p_i^-$ are variables equal to the inverse of the number of positively- and negatively-classified samples, respectively. We obtain those values by adding the following constraints to the underlying problem \eqref{eq:train_form}:
\begin{subequations}
\begin{align}
    & p_i^+ \le \hat{y_i} & i \in D \\
    & p_i^+ \le p_\text{ref}^+ & i \in D \\
    & p_i^+ \ge p_\text{ref}^+ - (1 - \hat{y}_i)& i \in D \\
    & p_i^+ \ge 0 \\
    & \sum_{i \in D} p_i^+ = 1 \\
    & p_i^- \le 1 - \hat{y_i} & i \in D \\
    & p_i^- \le p_\text{ref}^- & i \in D \\
    & p_i^- \ge p_\text{ref}^- - \hat{y}_i & i \in D \\
    & p_i^- \ge 0 \\
    & \sum_{i \in D} p_i^- = 1,
\end{align}
\end{subequations}
where we use the fact that the sum of the inverses over all positively (resp. negatively) classified samples must sum to 1 and all elements must be equal. This is achieved by making them equal to the reference variable $p_\text{ref}^+$ (resp. $p_\text{ref}^-$).

These constraints need to be added only once, since the values can be shared by multiple subgroups. Still, in our experiments, this formulation was usually not solved to a useful quality within the given time. 

\section{More results}
We expand the discussion, adding more results and insight into the framework. We discuss the scalability in Section \ref{app:time}, more results of the dense models in Section \ref{app:more}, including a deeper evaluation of the sparse models, and finally show results using DNFs as another predictor class in Section \ref{app:dnfs}. 


\paragraph{Conjunction interpretability}
To exemplify the usefulness of considering conjunctions in interpreting bias, we present the following comparison of identified subgroups in the ACS Travel Time dataset.

The SD maximizing conjunction reads:
\begin{equation*}
    \mathtt{sex = female} \land \mathtt{race = white}
\end{equation*}
while the linear model can be concisely represented as
\begin{equation*}
\begin{aligned}
    0.39x_{1} & + 0.49x_{2} - 0.28x_{3} + 0.18x_{4} - 1.00x_{5} \\
    & - 0.23x_{6} - 0.23x_{7} - 0.22x_{8} - 0.39x_{9} \\
    & - 0.56x_{10} - 0.24x_{11} - 0.05x_{12} - 0.55x_{13} \\
    & - 1.00x_{14} - 1.00x_{15} - 0.80x_{16} - 0.93x_{17} \\
    & - 1.00x_{18} \ge -1.12
\end{aligned}
\end{equation*}
where all parameters were rounded to 2 decimal places, possibly losing nuance in some combination of features. It certainly is not immediately clear which samples belong to the group, except that if a subject has any single feature equal to 1, they belong to the class. This is, however, unlikely, since some of the features are one-hot encoded, meaning at least a few must always be 1. 

\subsection{Scalability analysis}
\label{app:time} 
To assess scalability, we use the same detection setup with a tighter deadline of 5 minutes; see Figure \ref{fig:5min}. Within that time, even bias detection using conjunctions sometimes fails to find an optimal solution. It is, however, enough to consider 10 minutes to obtain the same results (for the conjunction subgroups) as in the main Figure \ref{fig:eval}. For results after 10 minutes, see Figure \ref{fig:10min}.

In both cases, there are datasets where MIO finds a more violating subgroup than GerryFair. In the case of ACS Income, our approach is even faster than GerryFair when considering conjunction subgroups, which is probably due to the low number of protected attributes.

\begin{figure*}
    \centering
    \includegraphics[width=\linewidth]{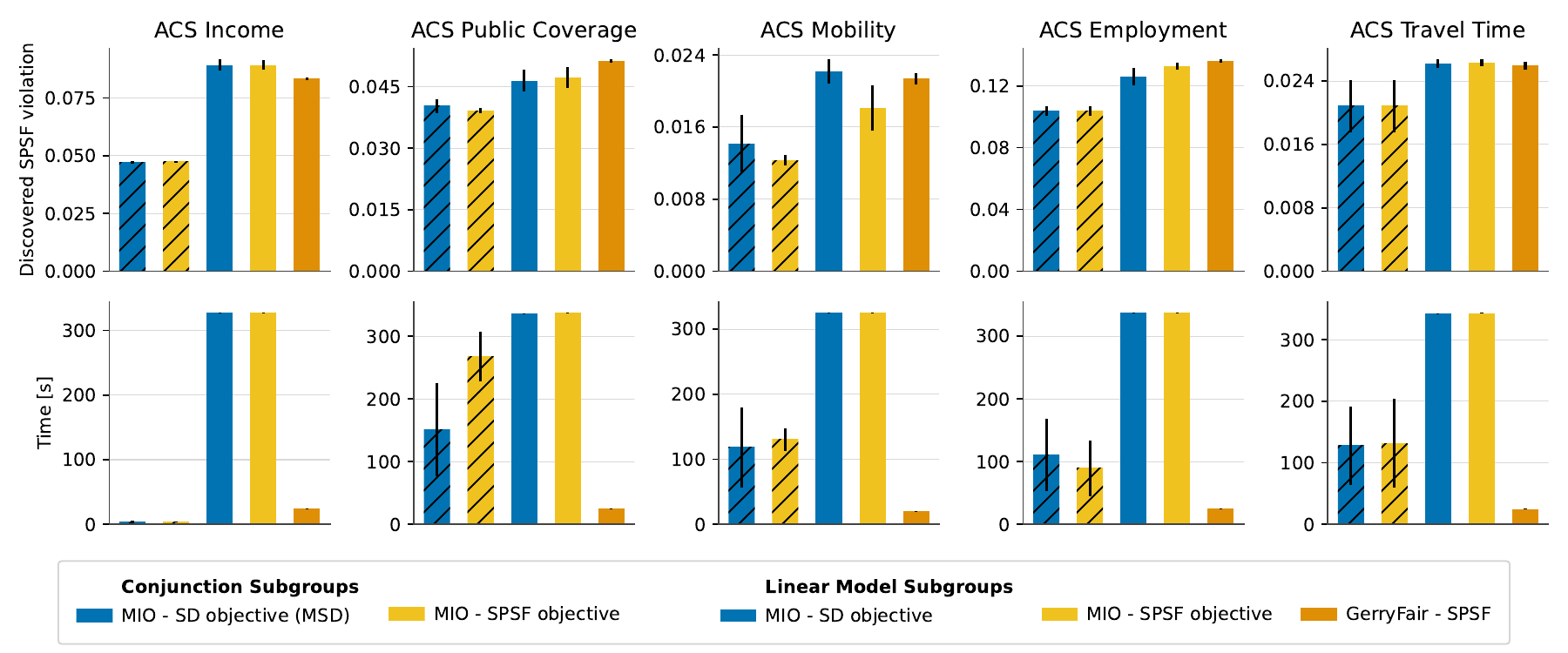}
    \caption{Results of bias detection with a time limit of 5 minutes.}
    \label{fig:5min}
\end{figure*}

\begin{figure*}
    \centering
    \includegraphics[width=\linewidth]{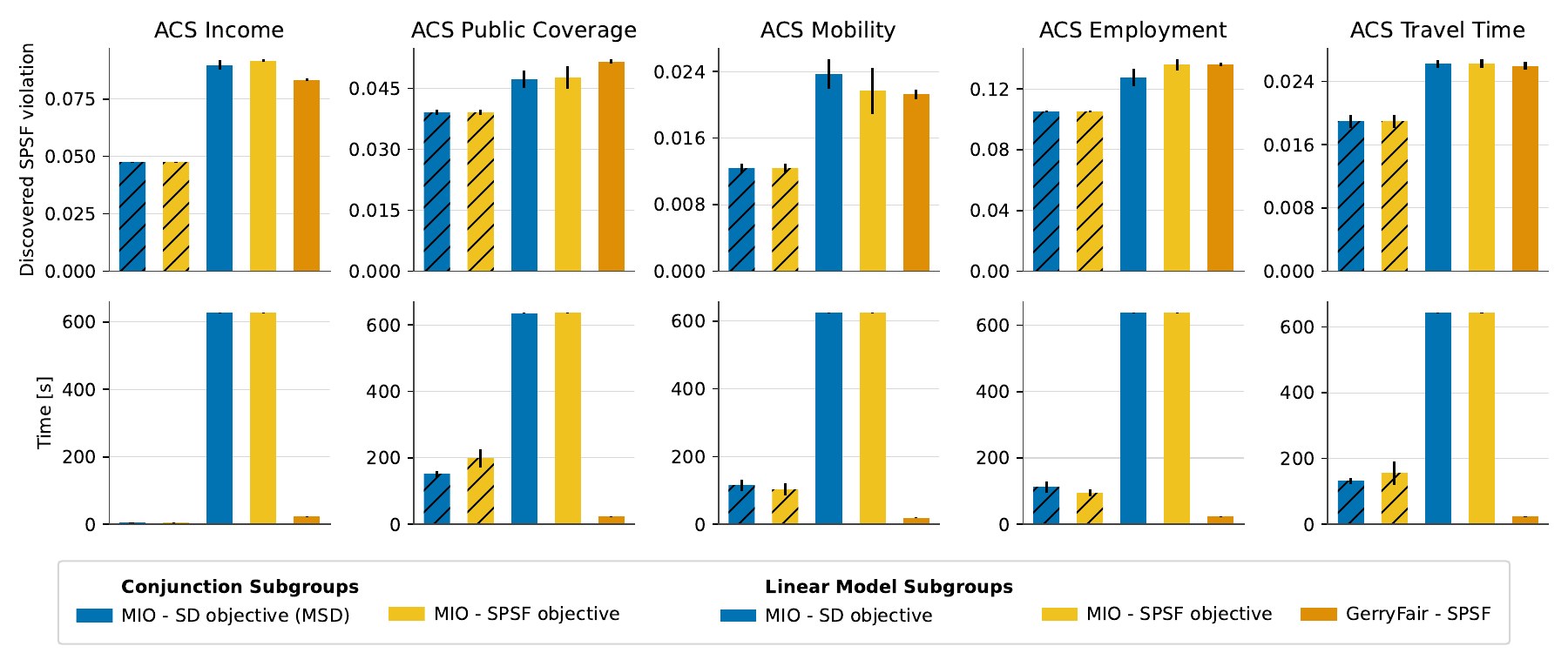}
    \caption{Results of bias detection with a time limit of 10 minutes.}
    \label{fig:10min}
\end{figure*}

\paragraph{Larger dataset} 
To showcase the scalability to larger amounts of data, we detect bias over the entire datasets, with a time budget of 24 hours. Results (Figure \ref{fig:1day}) are similar to the main results of the paper. Our framework usually finds subgroups with higher bias than GerryFair. 

\begin{figure*}
    \centering
    \includegraphics[width=\linewidth]{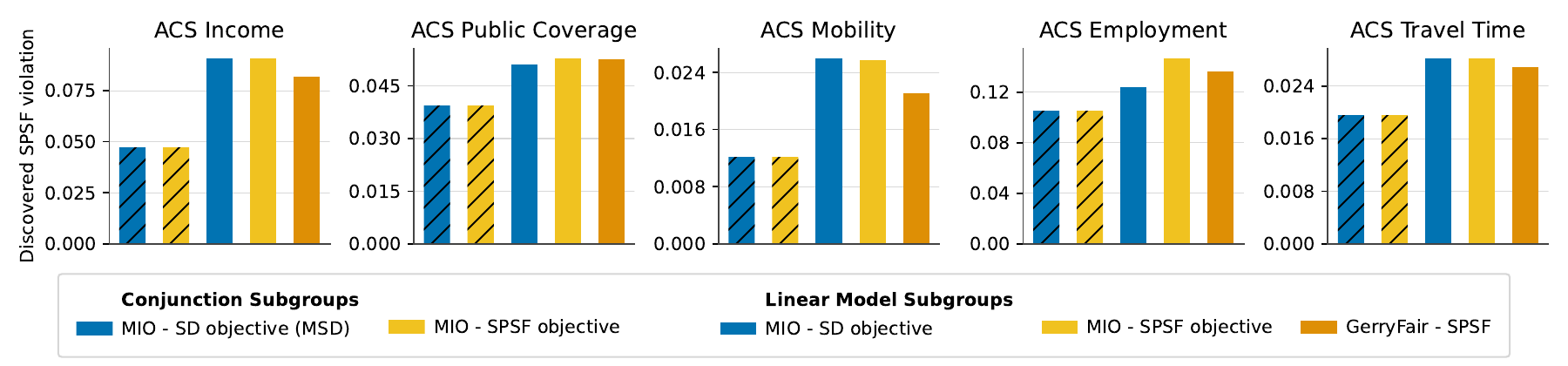}
    \caption{Results of bias detection over the entire datasets.}
    \label{fig:1day}
\end{figure*}

Our approach also rarely needs the computational budget provided. In fact, the solver finds a subgroup with bias close to the final value (within 98\% after a few minutes). See Figure \ref{fig:progress}, where we report the solver progress as the incumbent detected bias relative to its final value, in detections with linear subgroups. After a few minutes, the values are very close to the final value, and at 10 minutes, the median value is already indistinguishable from the final value. 

We can thus conclude that even on all of the data provided (up to around 370k samples, see Table \ref{tab:data}), our method is capable to outperform the GerryFair method in detection, within few minutes, in the median. Even in the biggest dataset with the highest unfairness, ACSEmployment, it takes around 1.3 hours for our approach to find a subgroup with more bias, well within a reasonable time budget for auditing. 

For conjunction subgroups, this might be even clearer, as those were shown to prove optimality in less time on similar data, in some cases, at least for the MSD objective \cite{nemecek_bias_2025}.

\begin{figure*}
    \centering
    \includegraphics[width=\linewidth]{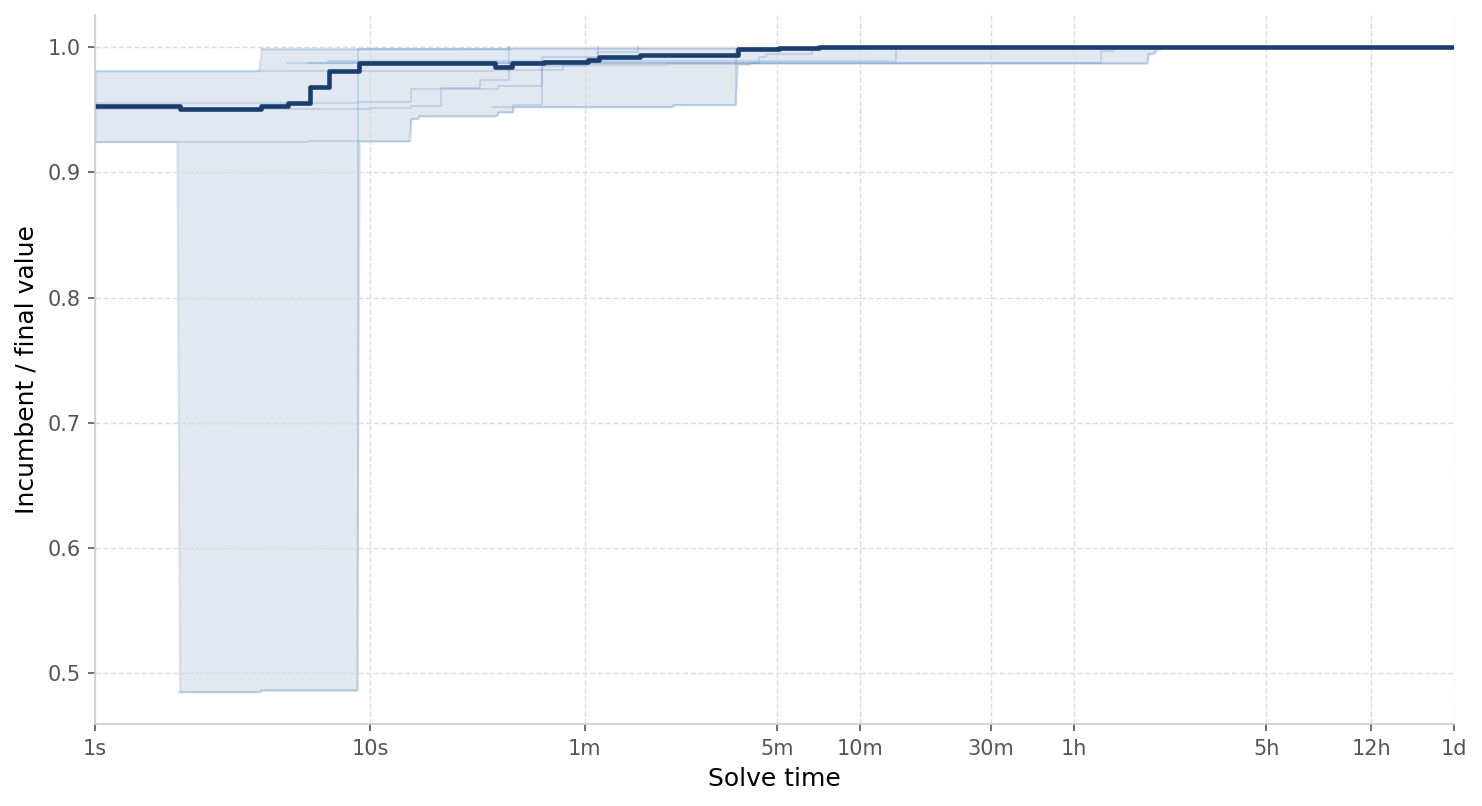}
    \caption{Solver process in bias detection over the entire datasets, over 10 runs with linear subgroups. Values represent the incumbent detected bias, relative to the final discovered value. The thick line is median value, thinner lines denote single runs and ribbon covers the range of all values. Note that x axis is log-scale. First values are available after a few seconds in some runs, leading to some jaggedness of the line and the ribbon.}
    \label{fig:progress}
\end{figure*}

\subsection{More results}
\label{app:more}
To allow for a more thorough comparison, we provide more detailed results in Figure \ref{fig:dense_big}. We include the F$_1$ score, the number of cuts, and the time taken. 
Note that GerryFair takes considerably more time during optimization, while returning less reliable results.
We also include the SPSF and MSD measures. The SPSF generally works well, often satisfying $\gamma$-FPSF even with conjunction subgroups. On the other hand, the MSD struggles to meet the fairness threshold, set to $4\gamma$ for the MSD, to be comparable to the SPSF measure. This choice corresponds to $\gamma$ for SPSF exactly when the models have balanced classification, and is a lower bound otherwise. The underperformance of MSD may be partly caused by this overconstraining of the problem, but the main influence is the notably greater difficulty of solving the formulation with the $O(|D|)$ additional constraints introduced. This is evidenced by the low number of cuts generated. The optimization struggles to provide a feasible solution to even start the cut generation process.

\begin{figure*}
    \centering
    \includegraphics[width=\linewidth]{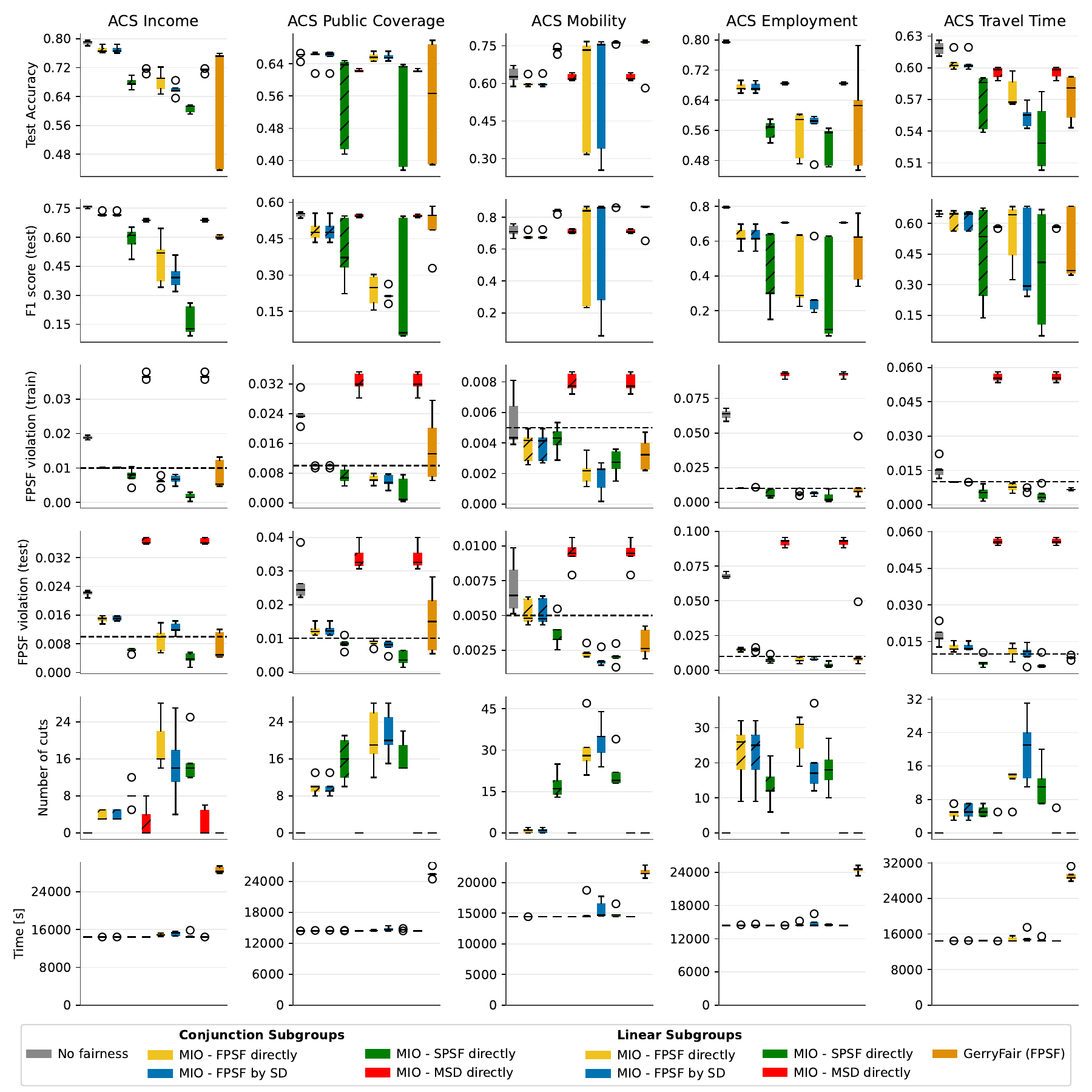}
    \caption{Detailed results of training a linear classifier without a sparsity term ($\sigma = 0$) under the proposed framework.}
    \label{fig:dense_big}
\end{figure*}

\subsubsection{Results of sparse models}
\label{app:sparse} 
We report the same set of evaluation dimensions also for the sparse linear models in Figure \ref{fig:sparse_big}. Note especially the decrease in unfairness and its relation to the number of cuts. For example, on ACS Income, more cuts were generated, leading to $\gamma$-FPSF satisfaction even on the training set for linear subgroups. 

\begin{figure*}
    \centering
    \includegraphics[width=\linewidth]{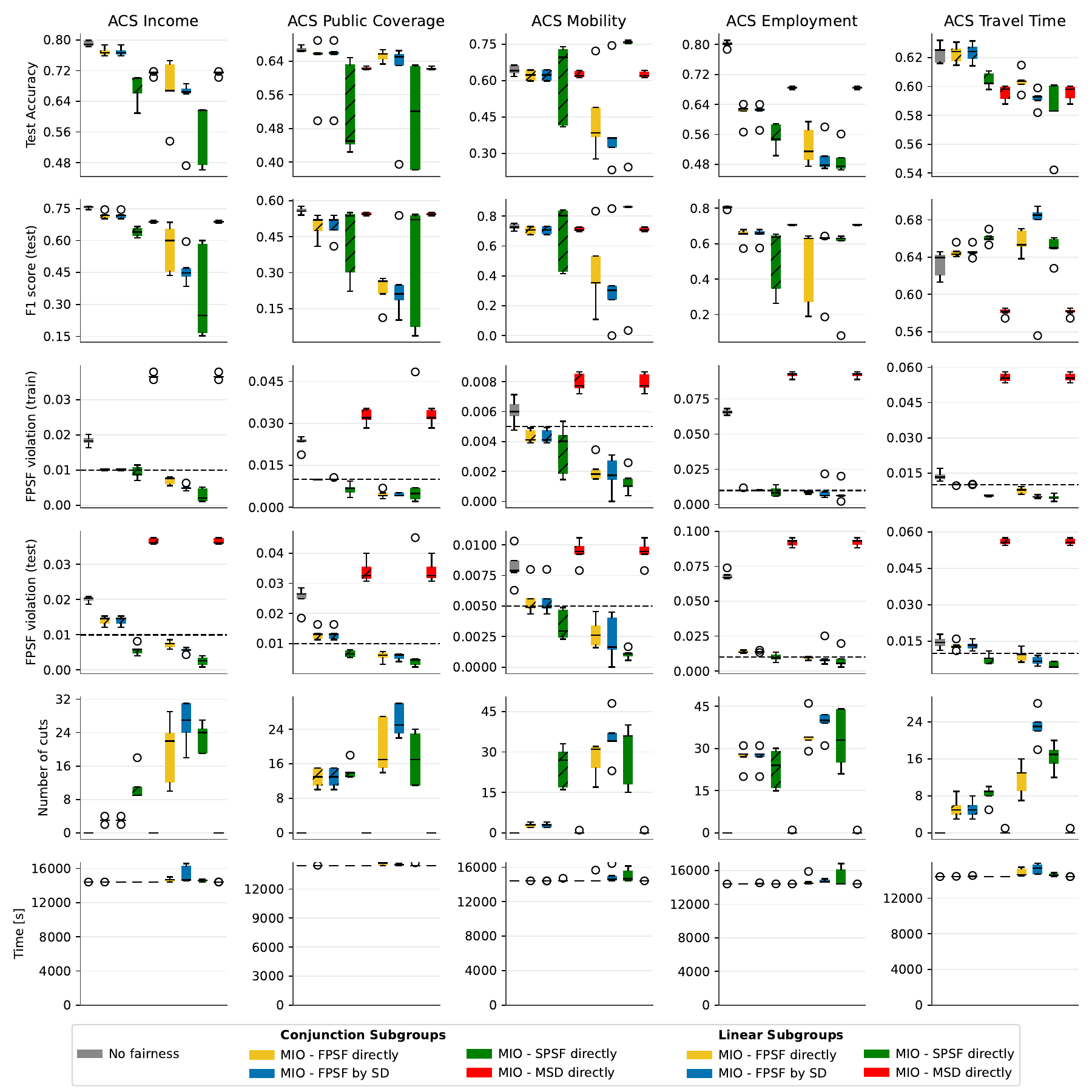}
    \caption{Detailed results of training a linear classifier with a sparsity term $\sigma = 0.1$ under the proposed framework}
    \label{fig:sparse_big}
\end{figure*}

\subsection{DNFs}
\label{app:dnfs} 
\begin{figure*}
    \centering
    \includegraphics[width=\linewidth]{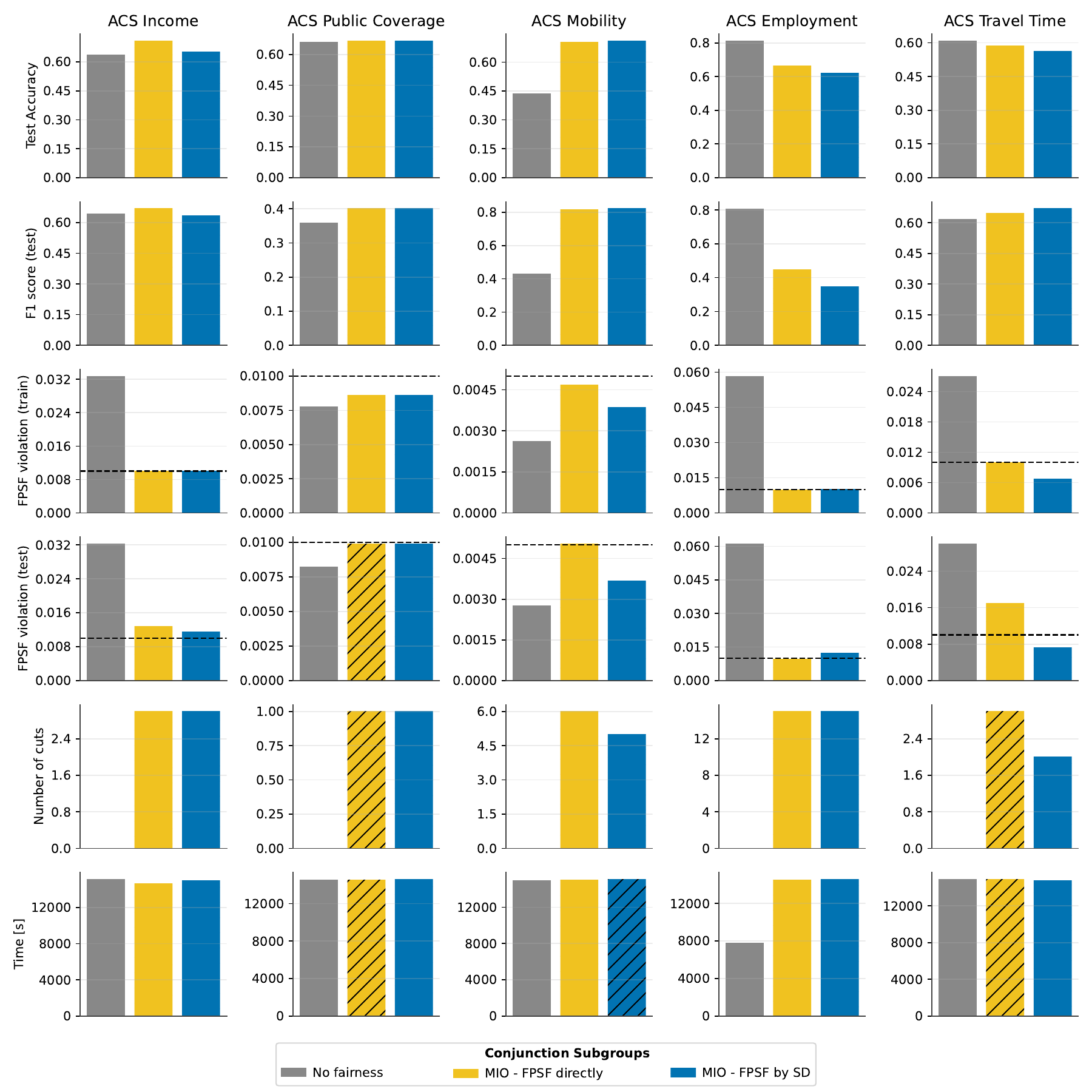}
    \caption{Preliminary results of training a DNF classifier with 1 term (a conjunction) under the proposed framework.}
    \label{fig:dnf1}
\end{figure*}

\begin{figure*}
    \centering
    \includegraphics[width=\linewidth]{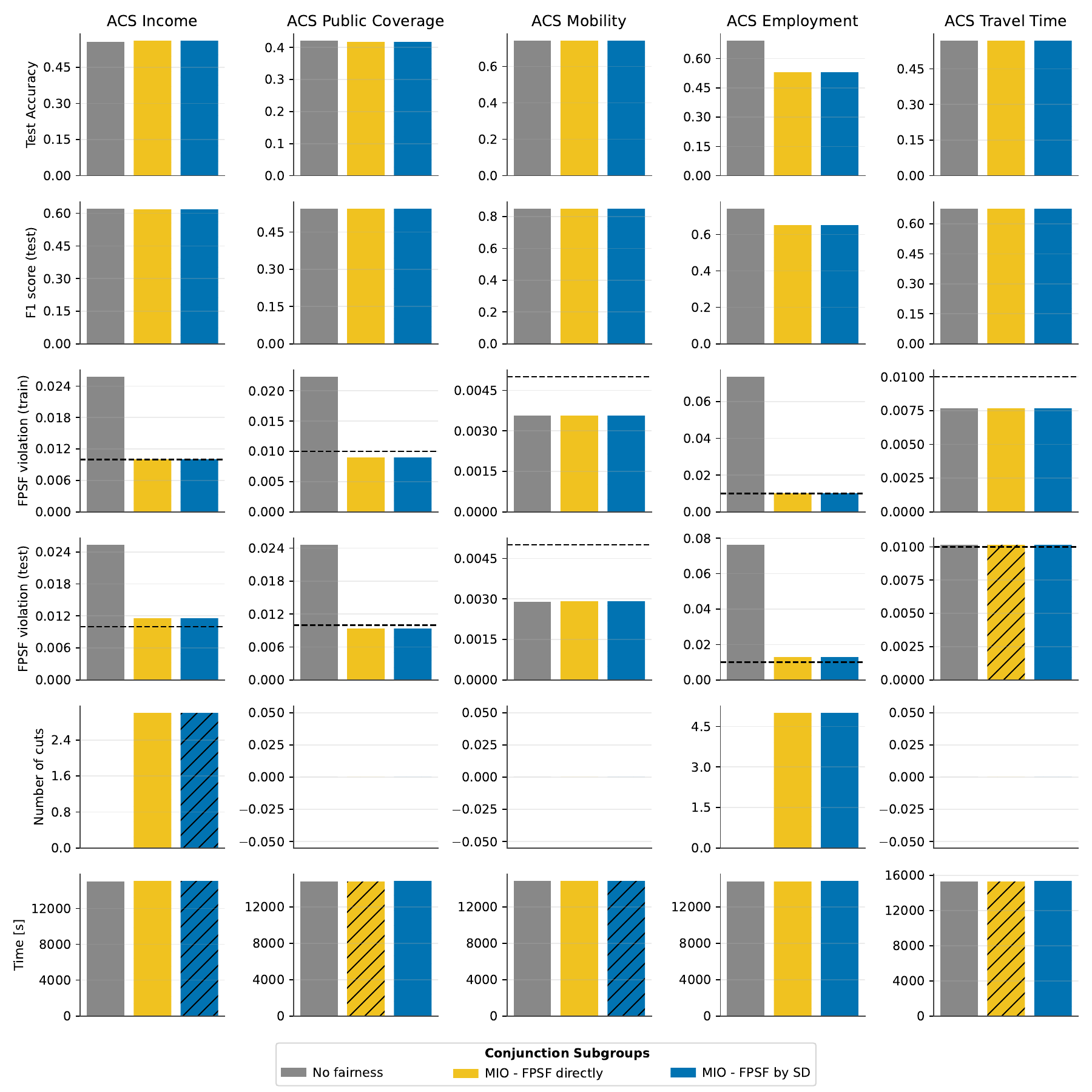}
    \caption{Preliminary results of training a DNF classifier with 3 terms under the proposed framework.}
    \label{fig:dnf3}
\end{figure*}

To showcase generality, we report preliminary results (single seed, otherwise the same setup) in training a DNF model, with 3 terms, and a conjunction (DNF with 1 term).
The preliminary results show that the method is indeed extensible to other models. Especially in the single-term DNF, the results are promising. We even observe an increase in performance after introducing fairness constraints, suggesting better generalization. At the same time, with increasing the number of terms, we encounter the limitations of using the straightforward formulation \eqref{eq:dnf}. The time budget is clearly insufficient, as the solver at times struggles to even find a well-performing model. Here, extending the formulation to the column-generation approach  \citep{dash_boolean_2018}, for example, might improve the scalability of training the predictor itself, while allowing for the fairness cut generation proposed here.

\section{Limitations}
\paragraph{Scalability} Our framework relies on Mixed-Integer Optimization, which is inherently limited in scalability. While the lazy constraint approach substantially reduces the number of fairness constraints that need to be explicitly considered, training complex classifiers remains computationally challenging within practical time budgets. Our work focuses on the conceptual proposition of such an approach, exemplified using linear models and DNFs. It should be noted that the detection part scales much better on its own, so a different optimization method might be used for the training problem, including gradient-based methods.
\paragraph{Discrete protected attributes} The conjunction-based subgroup class assumes all protected attributes are discrete. While many protected attributes are naturally discrete (e.g., race, sex), others, such as age or income, require discretization. The choice of bin boundaries is non-trivial and can itself influence which subgroups are identified as most disadvantaged. In practice, meaningful thresholds often exist (e.g., age cutoffs relevant to employment law). Despite that, we acknowledge that a principled approach to discretization within this framework remains an open problem.
\paragraph{Out-of-sample fairness} Our MIO-based training provides a strict fairness guarantee on the training data. However, as with any constrained learning method, this guarantee does not directly transfer to held-out data. In our experiments, test-set FPSF violations occasionally exceeded the threshold $\gamma$ by a modest margin. At the same time, theoretical generalization bounds for MSD and SPSF/FPSF do exist. This might be particularly useful for deployment in regulated settings where out-of-sample guarantees may be required.
\paragraph{Equivalence scope} The theoretical equivalence between SPSF and MSD (Corollary 1.1) holds specifically for the task of detecting the most unfair subgroup. The value of the two measures can differ substantially under class imbalance. Our experiments show that MSD as a direct training objective underperforms SPSF/FPSF, partly due to the additional $O(n)$ variables and constraints it introduces into the optimization. Practitioners should therefore treat the equivalence result as a tool for auditing rather than a general interchangeability claim.


\clearpage
\newpage

\section*{NeurIPS Paper Checklist}

\begin{enumerate}

\item {\bf Claims}
    \item[] Question: Do the main claims made in the abstract and introduction accurately reflect the paper's contributions and scope?
    \item[] Answer: \answerYes{} 
    \item[] Justification: We provide the proof and empirically compare the performance showing clear improvement over SotA in bias detection.
    \item[] Guidelines:
    \begin{itemize}
        \item The answer \answerNA{} means that the abstract and introduction do not include the claims made in the paper.
        \item The abstract and/or introduction should clearly state the claims made, including the contributions made in the paper and important assumptions and limitations. A \answerNo{} or \answerNA{} answer to this question will not be perceived well by the reviewers. 
        \item The claims made should match theoretical and experimental results, and reflect how much the results can be expected to generalize to other settings. 
        \item It is fine to include aspirational goals as motivation as long as it is clear that these goals are not attained by the paper. 
    \end{itemize}

\item {\bf Limitations}
    \item[] Question: Does the paper discuss the limitations of the work performed by the authors?
    \item[] Answer: \answerYes{} 
    \item[] Justification: Limitations are discussed throughout the paper, with a more explicit discussion in the appendix. 
    \item[] Guidelines:
    \begin{itemize}
        \item The answer \answerNA{} means that the paper has no limitation while the answer \answerNo{} means that the paper has limitations, but those are not discussed in the paper. 
        \item The authors are encouraged to create a separate ``Limitations'' section in their paper.
        \item The paper should point out any strong assumptions and how robust the results are to violations of these assumptions (e.g., independence assumptions, noiseless settings, model well-specification, asymptotic approximations only holding locally). The authors should reflect on how these assumptions might be violated in practice and what the implications would be.
        \item The authors should reflect on the scope of the claims made, e.g., if the approach was only tested on a few datasets or with a few runs. In general, empirical results often depend on implicit assumptions, which should be articulated.
        \item The authors should reflect on the factors that influence the performance of the approach. For example, a facial recognition algorithm may perform poorly when image resolution is low or images are taken in low lighting. Or a speech-to-text system might not be used reliably to provide closed captions for online lectures because it fails to handle technical jargon.
        \item The authors should discuss the computational efficiency of the proposed algorithms and how they scale with dataset size.
        \item If applicable, the authors should discuss possible limitations of their approach to address problems of privacy and fairness.
        \item While the authors might fear that complete honesty about limitations might be used by reviewers as grounds for rejection, a worse outcome might be that reviewers discover limitations that aren't acknowledged in the paper. The authors should use their best judgment and recognize that individual actions in favor of transparency play an important role in developing norms that preserve the integrity of the community. Reviewers will be specifically instructed to not penalize honesty concerning limitations.
    \end{itemize}

\item {\bf Theory assumptions and proofs}
    \item[] Question: For each theoretical result, does the paper provide the full set of assumptions and a complete (and correct) proof?
    \item[] Answer: \answerYes{} 
    \item[] Justification: The proof is written in specific steps with commentary right after the theorem statement.
    \item[] Guidelines:
    \begin{itemize}
        \item The answer \answerNA{} means that the paper does not include theoretical results. 
        \item All the theorems, formulas, and proofs in the paper should be numbered and cross-referenced.
        \item All assumptions should be clearly stated or referenced in the statement of any theorems.
        \item The proofs can either appear in the main paper or the supplemental material, but if they appear in the supplemental material, the authors are encouraged to provide a short proof sketch to provide intuition. 
        \item Inversely, any informal proof provided in the core of the paper should be complemented by formal proofs provided in appendix or supplemental material.
        \item Theorems and Lemmas that the proof relies upon should be properly referenced. 
    \end{itemize}

    \item {\bf Experimental result reproducibility}
    \item[] Question: Does the paper fully disclose all the information needed to reproduce the main experimental results of the paper to the extent that it affects the main claims and/or conclusions of the paper (regardless of whether the code and data are provided or not)?
    \item[] Answer: \answerYes{} 
    \item[] Justification: The MIO formulations that were optimized are described in the appendix, code is provided within supplementary material. (And will be made open source)
    \item[] Guidelines:
    \begin{itemize}
        \item The answer \answerNA{} means that the paper does not include experiments.
        \item If the paper includes experiments, a \answerNo{} answer to this question will not be perceived well by the reviewers: Making the paper reproducible is important, regardless of whether the code and data are provided or not.
        \item If the contribution is a dataset and\slash or model, the authors should describe the steps taken to make their results reproducible or verifiable. 
        \item Depending on the contribution, reproducibility can be accomplished in various ways. For example, if the contribution is a novel architecture, describing the architecture fully might suffice, or if the contribution is a specific model and empirical evaluation, it may be necessary to either make it possible for others to replicate the model with the same dataset, or provide access to the model. In general. releasing code and data is often one good way to accomplish this, but reproducibility can also be provided via detailed instructions for how to replicate the results, access to a hosted model (e.g., in the case of a large language model), releasing of a model checkpoint, or other means that are appropriate to the research performed.
        \item While NeurIPS does not require releasing code, the conference does require all submissions to provide some reasonable avenue for reproducibility, which may depend on the nature of the contribution. For example
        \begin{enumerate}
            \item If the contribution is primarily a new algorithm, the paper should make it clear how to reproduce that algorithm.
            \item If the contribution is primarily a new model architecture, the paper should describe the architecture clearly and fully.
            \item If the contribution is a new model (e.g., a large language model), then there should either be a way to access this model for reproducing the results or a way to reproduce the model (e.g., with an open-source dataset or instructions for how to construct the dataset).
            \item We recognize that reproducibility may be tricky in some cases, in which case authors are welcome to describe the particular way they provide for reproducibility. In the case of closed-source models, it may be that access to the model is limited in some way (e.g., to registered users), but it should be possible for other researchers to have some path to reproducing or verifying the results.
        \end{enumerate}
    \end{itemize}

\item {\bf Open access to data and code}
    \item[] Question: Does the paper provide open access to the data and code, with sufficient instructions to faithfully reproduce the main experimental results, as described in supplemental material?
    \item[] Answer: \answerYes{} 
    \item[] Justification: The code is attached and will be made open source.
    \item[] Guidelines:
    \begin{itemize}
        \item The answer \answerNA{} means that paper does not include experiments requiring code.
        \item Please see the NeurIPS code and data submission guidelines (\url{https://neurips.cc/public/guides/CodeSubmissionPolicy}) for more details.
        \item While we encourage the release of code and data, we understand that this might not be possible, so \answerNo{} is an acceptable answer. Papers cannot be rejected simply for not including code, unless this is central to the contribution (e.g., for a new open-source benchmark).
        \item The instructions should contain the exact command and environment needed to run to reproduce the results. See the NeurIPS code and data submission guidelines (\url{https://neurips.cc/public/guides/CodeSubmissionPolicy}) for more details.
        \item The authors should provide instructions on data access and preparation, including how to access the raw data, preprocessed data, intermediate data, and generated data, etc.
        \item The authors should provide scripts to reproduce all experimental results for the new proposed method and baselines. If only a subset of experiments are reproducible, they should state which ones are omitted from the script and why.
        \item At submission time, to preserve anonymity, the authors should release anonymized versions (if applicable).
        \item Providing as much information as possible in supplemental material (appended to the paper) is recommended, but including URLs to data and code is permitted.
    \end{itemize}

\item {\bf Experimental setting/details}
    \item[] Question: Does the paper specify all the training and test details (e.g., data splits, hyperparameters, how they were chosen, type of optimizer) necessary to understand the results?
    \item[] Answer: \answerYes{} 
    \item[] Justification: Experimental settings are described down to hardware specification, and code is provided within supplementary material. (And will be made open source)
    \item[] Guidelines:
    \begin{itemize}
        \item The answer \answerNA{} means that the paper does not include experiments.
        \item The experimental setting should be presented in the core of the paper to a level of detail that is necessary to appreciate the results and make sense of them.
        \item The full details can be provided either with the code, in appendix, or as supplemental material.
    \end{itemize}

\item {\bf Experiment statistical significance}
    \item[] Question: Does the paper report error bars suitably and correctly defined or other appropriate information about the statistical significance of the experiments?
    \item[] Answer: \answerYes{} 
    \item[] Justification: All plots contain error bars.
    \item[] Guidelines:
    \begin{itemize}
        \item The answer \answerNA{} means that the paper does not include experiments.
        \item The authors should answer \answerYes{} if the results are accompanied by error bars, confidence intervals, or statistical significance tests, at least for the experiments that support the main claims of the paper.
        \item The factors of variability that the error bars are capturing should be clearly stated (for example, train/test split, initialization, random drawing of some parameter, or overall run with given experimental conditions).
        \item The method for calculating the error bars should be explained (closed form formula, call to a library function, bootstrap, etc.)
        \item The assumptions made should be given (e.g., Normally distributed errors).
        \item It should be clear whether the error bar is the standard deviation or the standard error of the mean.
        \item It is OK to report 1-sigma error bars, but one should state it. The authors should preferably report a 2-sigma error bar than state that they have a 96\% CI, if the hypothesis of Normality of errors is not verified.
        \item For asymmetric distributions, the authors should be careful not to show in tables or figures symmetric error bars that would yield results that are out of range (e.g., negative error rates).
        \item If error bars are reported in tables or plots, the authors should explain in the text how they were calculated and reference the corresponding figures or tables in the text.
    \end{itemize}

\item {\bf Experiments compute resources}
    \item[] Question: For each experiment, does the paper provide sufficient information on the computer resources (type of compute workers, memory, time of execution) needed to reproduce the experiments?
    \item[] Answer: \answerYes{} 
    \item[] Justification: Compute resources are described, including the hardware specification, and the implementation is provided within supplementary material.
    \item[] Guidelines:
    \begin{itemize}
        \item The answer \answerNA{} means that the paper does not include experiments.
        \item The paper should indicate the type of compute workers CPU or GPU, internal cluster, or cloud provider, including relevant memory and storage.
        \item The paper should provide the amount of compute required for each of the individual experimental runs as well as estimate the total compute. 
        \item The paper should disclose whether the full research project required more compute than the experiments reported in the paper (e.g., preliminary or failed experiments that didn't make it into the paper). 
    \end{itemize}
    
\item {\bf Code of ethics}
    \item[] Question: Does the research conducted in the paper conform, in every respect, with the NeurIPS Code of Ethics \url{https://neurips.cc/public/EthicsGuidelines}?
    \item[] Answer: \answerYes{} 
    \item[] Justification: The proposed method is designed to combat algorithmic bias, aiming for societal benefit. Societal impact and potential harms are discussed in an ethics statement in the appendix. The dataset used is publicly available, and all code will be open-sourced under an appropriate license.
    \item[] Guidelines:
    \begin{itemize}
        \item The answer \answerNA{} means that the authors have not reviewed the NeurIPS Code of Ethics.
        \item If the authors answer \answerNo, they should explain the special circumstances that require a deviation from the Code of Ethics.
        \item The authors should make sure to preserve anonymity (e.g., if there is a special consideration due to laws or regulations in their jurisdiction).
    \end{itemize}

\item {\bf Broader impacts}
    \item[] Question: Does the paper discuss both potential positive societal impacts and negative societal impacts of the work performed?
    \item[] Answer: \answerYes{} 
    \item[] Justification: The paper includes such statement in the appendix.
    \item[] Guidelines:
    \begin{itemize}
        \item The answer \answerNA{} means that there is no societal impact of the work performed.
        \item If the authors answer \answerNA{} or \answerNo, they should explain why their work has no societal impact or why the paper does not address societal impact.
        \item Examples of negative societal impacts include potential malicious or unintended uses (e.g., disinformation, generating fake profiles, surveillance), fairness considerations (e.g., deployment of technologies that could make decisions that unfairly impact specific groups), privacy considerations, and security considerations.
        \item The conference expects that many papers will be foundational research and not tied to particular applications, let alone deployments. However, if there is a direct path to any negative applications, the authors should point it out. For example, it is legitimate to point out that an improvement in the quality of generative models could be used to generate Deepfakes for disinformation. On the other hand, it is not needed to point out that a generic algorithm for optimizing neural networks could enable people to train models that generate Deepfakes faster.
        \item The authors should consider possible harms that could arise when the technology is being used as intended and functioning correctly, harms that could arise when the technology is being used as intended but gives incorrect results, and harms following from (intentional or unintentional) misuse of the technology.
        \item If there are negative societal impacts, the authors could also discuss possible mitigation strategies (e.g., gated release of models, providing defenses in addition to attacks, mechanisms for monitoring misuse, mechanisms to monitor how a system learns from feedback over time, improving the efficiency and accessibility of ML).
    \end{itemize}
    
\item {\bf Safeguards}
    \item[] Question: Does the paper describe safeguards that have been put in place for responsible release of data or models that have a high risk for misuse (e.g., pre-trained language models, image generators, or scraped datasets)?
    \item[] Answer: \answerNA{} 
    \item[] Justification: Paper poses no such risks.
    \item[] Guidelines:
    \begin{itemize}
        \item The answer \answerNA{} means that the paper poses no such risks.
        \item Released models that have a high risk for misuse or dual-use should be released with necessary safeguards to allow for controlled use of the model, for example by requiring that users adhere to usage guidelines or restrictions to access the model or implementing safety filters. 
        \item Datasets that have been scraped from the Internet could pose safety risks. The authors should describe how they avoided releasing unsafe images.
        \item We recognize that providing effective safeguards is challenging, and many papers do not require this, but we encourage authors to take this into account and make a best faith effort.
    \end{itemize}

\item {\bf Licenses for existing assets}
    \item[] Question: Are the creators or original owners of assets (e.g., code, data, models), used in the paper, properly credited and are the license and terms of use explicitly mentioned and properly respected?
    \item[] Answer: \answerYes{} 
    \item[] Justification: The dataset is not our own, its MIT license is explicitly stated in the paper.
    \item[] Guidelines:
    \begin{itemize}
        \item The answer \answerNA{} means that the paper does not use existing assets.
        \item The authors should cite the original paper that produced the code package or dataset.
        \item The authors should state which version of the asset is used and, if possible, include a URL.
        \item The name of the license (e.g., CC-BY 4.0) should be included for each asset.
        \item For scraped data from a particular source (e.g., website), the copyright and terms of service of that source should be provided.
        \item If assets are released, the license, copyright information, and terms of use in the package should be provided. For popular datasets, \url{paperswithcode.com/datasets} has curated licenses for some datasets. Their licensing guide can help determine the license of a dataset.
        \item For existing datasets that are re-packaged, both the original license and the license of the derived asset (if it has changed) should be provided.
        \item If this information is not available online, the authors are encouraged to reach out to the asset's creators.
    \end{itemize}

\item {\bf New assets}
    \item[] Question: Are new assets introduced in the paper well documented and is the documentation provided alongside the assets?
    \item[] Answer: \answerYes{} 
    \item[] Justification: The paper introduces only the implementation (code) which will be made open source under CC-BY. This is stated in the paper.
    \item[] Guidelines:
    \begin{itemize}
        \item The answer \answerNA{} means that the paper does not release new assets.
        \item Researchers should communicate the details of the dataset\slash code\slash model as part of their submissions via structured templates. This includes details about training, license, limitations, etc. 
        \item The paper should discuss whether and how consent was obtained from people whose asset is used.
        \item At submission time, remember to anonymize your assets (if applicable). You can either create an anonymized URL or include an anonymized zip file.
    \end{itemize}

\item {\bf Crowdsourcing and research with human subjects}
    \item[] Question: For crowdsourcing experiments and research with human subjects, does the paper include the full text of instructions given to participants and screenshots, if applicable, as well as details about compensation (if any)? 
    \item[] Answer: \answerNA{} 
    \item[] Justification: The paper does not involve crowdsourcing nor research with human subjects.
    \item[] Guidelines:
    \begin{itemize}
        \item The answer \answerNA{} means that the paper does not involve crowdsourcing nor research with human subjects.
        \item Including this information in the supplemental material is fine, but if the main contribution of the paper involves human subjects, then as much detail as possible should be included in the main paper. 
        \item According to the NeurIPS Code of Ethics, workers involved in data collection, curation, or other labor should be paid at least the minimum wage in the country of the data collector. 
    \end{itemize}

\item {\bf Institutional review board (IRB) approvals or equivalent for research with human subjects}
    \item[] Question: Does the paper describe potential risks incurred by study participants, whether such risks were disclosed to the subjects, and whether Institutional Review Board (IRB) approvals (or an equivalent approval/review based on the requirements of your country or institution) were obtained?
    \item[] Answer: \answerNA{} 
    \item[] Justification: The paper does not involve crowdsourcing nor research with human subjects.
    \item[] Guidelines:
    \begin{itemize}
        \item The answer \answerNA{} means that the paper does not involve crowdsourcing nor research with human subjects.
        \item Depending on the country in which research is conducted, IRB approval (or equivalent) may be required for any human subjects research. If you obtained IRB approval, you should clearly state this in the paper. 
        \item We recognize that the procedures for this may vary significantly between institutions and locations, and we expect authors to adhere to the NeurIPS Code of Ethics and the guidelines for their institution. 
        \item For initial submissions, do not include any information that would break anonymity (if applicable), such as the institution conducting the review.
    \end{itemize}

\item {\bf Declaration of LLM usage}
    \item[] Question: Does the paper describe the usage of LLMs if it is an important, original, or non-standard component of the core methods in this research? Note that if the LLM is used only for writing, editing, or formatting purposes and does \emph{not} impact the core methodology, scientific rigor, or originality of the research, declaration is not required.
    \item[] Answer: \answerNA{} 
    \item[] Justification: LLMs were used mainly for text flow, aiding in result visualization, and preliminary review. All uses were thoroughly checked, and the authors bear full responsibility.  
    \item[] Guidelines:
    \begin{itemize}
        \item The answer \answerNA{} means that the core method development in this research does not involve LLMs as any important, original, or non-standard components.
        \item Please refer to our LLM policy in the NeurIPS handbook for what should or should not be described.
    \end{itemize}

\end{enumerate}

\end{document}